%% file: main_arxiv.tex
\documentclass[twocolumn]{fairmeta} 

\usepackage[T1]{fontenc}    
\usepackage{tcolorbox}
\tcbuselibrary{skins}
\usepackage{enumitem}
\usepackage{xspace}
\makeatletter
\newcommand{\maketitleTwo}{
  \if@twocolumn
    \twocolumn[%
      \mymaketitle
      \vskip 0.38cm
      \input{teaserfig_arxiv}
      \vskip 0.38cm
    ]%
  \else
    \mymaketitle
    \par\vspace{0.38cm}
    \input{teaserfig_arxiv}
  \fi
}
\makeatother

\newif\ifcvpr
\cvprfalse

\input{preamble}

\author[1,2]{Tariq Berrada Ifriqi}
\author[1]{John Nguyen}
\author[2]{Karteek Alahari}
\author[1]{Jakob Verbeek}
\author[1]{Ricky T. Q. Chen}

\affiliation[1]{FAIR at Meta}
\affiliation[2]{Univ.\ Grenoble Alpes, Inria, CNRS, Grenoble INP, LJK, France}

\abstract{\input{sections/abstract}}

\correspondence{Tariq Berrada Ifriqi at \email{tariqberrada@outlook.com}}
\metadata[Code]{\url{https://github.com/facebookresearch/flowception}}
\metadata[Website]{\href{https://flowception-meta.github.io}{flowception-meta.github.io}}

\begin{document}

\maketitleTwo

\input{sections/introduction}
\input{sections/related_work}

\input{sections/setup}
\input{sections/method2}
\input{sections/experiments}
\input{sections/conclusion}

\mypar{Acknowledgements}
Karteek Alahari was supported in part by the Institute of Information \& Communications Technology Planning \& Evaluation (IITP) grant funded by the Korean Government (MSIT) (No.\ RS-2024-00457882, National AI Research Lab Project).
\clearpage

\bibliographystyle{assets/plainnat}
\bibliography{paper}

\clearpage
\newpage
\beginappendix

\input{sections/appendix}

\end{document}

%% file: teaserfig_arxiv.tex
\begin{center}
    \captionsetup{type=figure}
    \vspace{-4mm}
    \includegraphics[width=\linewidth]{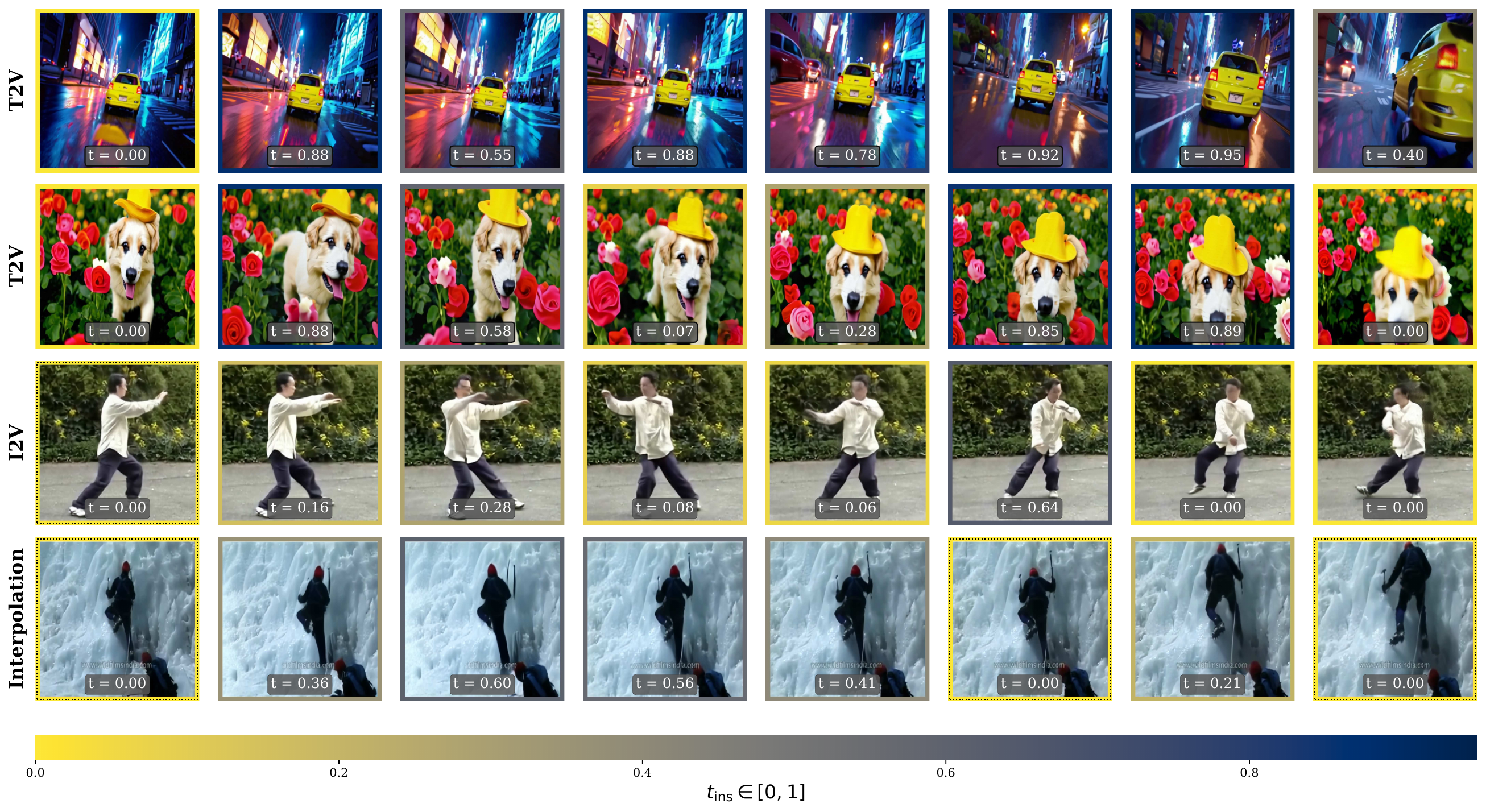}
    \vspace{-7mm}
    \caption{{\bf Example generations for \ours under different modalities.}
      Context frames marked by dashed boundaries.
      \ours enables variable-length non-autoregressive generation by learning to both denoise and insert frames in any order as sampling progresses (captions in the supplementary).}
    \label{fig:teaser}
    \vspace{0mm}
\end{center}

%% file: preamble.tex
\usepackage{algorithm}
\usepackage{algpseudocode}
\usepackage{amsthm}
\usepackage{mathtools}
\usepackage{bbm}
\usepackage{tcolorbox}
\usepackage{diagbox}
\usepackage{pifont}
\usepackage{amsmath}
\usepackage{amssymb}
\usepackage{sidecap}
\usepackage{adjustbox}
\usepackage{xcolor}
\usepackage{colortbl}

\title{Flowception: Temporally Expansive Flow Matching for Video Generation}

\newcommand{\ours}{Flowception\xspace}

\definecolor{EditTeal}{RGB}{5, 146, 148}
\colorlet{mygreen}{green!55!black}
\definecolor{myyellow1}{HTML}{D89A3C}
\colorlet{myyellow}{myyellow1!90!black}
\definecolor{metablue}{HTML}{0064E0}
\definecolor{cvprblue}{rgb}{0.21,0.49,0.74}
\definecolor{colQuality}{RGB}{240, 248, 255}  
\definecolor{colFVD}{RGB}{255, 250, 240}      

\newcommand{\cellhi}{\cellcolor{metablue!10}}

\def\mypar#1{\vspace{2mm}\noindent\textbf{#1.}\hspace{2mm}}

\def\tb#1{}
\def\jkb#1{}
\def\kar#1{}
\def\ricky#1{}
\def\todo#1{}
\def\john#1{}

\newcommand{\ptarget}{p_\text{target}}

\newcommand{\1}{\mathbbm{1}}

\newcommand{\X}{{X}}

\newcommand{\ins}{{{\text{ins}}}}

\newcommand{\strip}{f_{\text{strip}}}
\newcommand{\blank}{\varnothing}
\newcommand{\Z}{{Z}}
\newcommand{\E}{\mathbb{E}}

\makeatletter
\DeclareRobustCommand\onedot{\futurelet\@let@token\@onedot}
\def\@onedot{\ifx\@let@token.\else.\null\fi\xspace}
\makeatother

\def\eg{e.g\onedot} 
\def\ie{i.e\onedot} 
   
 \def\vs{vs\onedot}
\def\wrt{w.r.t\onedot}

\DeclareMathOperator{\clip}{clip}

\theoremstyle{plain}  
\newtheorem{theorem}{Theorem}[section]  
\newtheorem{cor}[theorem]{Corollary}
\newtheorem{lem}[theorem]{Lemma}
\newtheorem{prop}[theorem]{Proposition}

%% file: sections/abstract.tex
We present Flowception, a novel non-autoregressive and variable-length video generation framework. Flowception learns a probability path that interleaves discrete frame insertions with continuous frame denoising. Compared to autoregressive methods, Flowception alleviates error accumulation/drift as the frame insertion mechanism during sampling serves as an efficient compression mechanism to handle long-term context. Compared to full-sequence flows, our method reduces FLOPs for training three-fold, while also being more amenable to local attention variants, and allowing to learn the length of videos jointly with their content. Quantitative experimental results show improved FVD and VBench metrics over autoregressive and full-sequence baselines, which is further validated with qualitative results. Finally, by learning to insert and denoise frames in a sequence, Flowception seamlessly integrates different tasks such as image-to-video generation and video interpolation.

%% file: sections/introduction.tex
\vspace{-2mm}
\section{Introduction}
\vspace{-2mm}
Recent advances in diffusion and flow-matching have unlocked high-fidelity image generation~\citep{sd3, flux2024} and are rapidly migrating to video. 
Current video generation models typically adopt one of two paradigms: \emph{full-sequence} generation denoises all frames jointly with full attention \citep{wan2025, HaCohen2024LTXVideo, opensora}, while  temporal \emph{autoregressive} (AR) generation produces frames (or blocks of frames) sequentially in a left-to-right order~\citep{genie, song2025historyguidedvideodiffusion}. 
Full-sequence models benefit from bidirectional attention, enabling the model to correct errors during denoising and achieving superior generation quality. 
However, parallel denoising of the entire sequence prevents real-time streaming, as frames cannot be returned until fully denoised.  
Moreover, full-sequence models use a fixed generation length and incur a quadratic attention cost in the number of frames, limiting long-term generation.

\begin{figure}
    \centering
    \includegraphics[width=.5\textwidth]{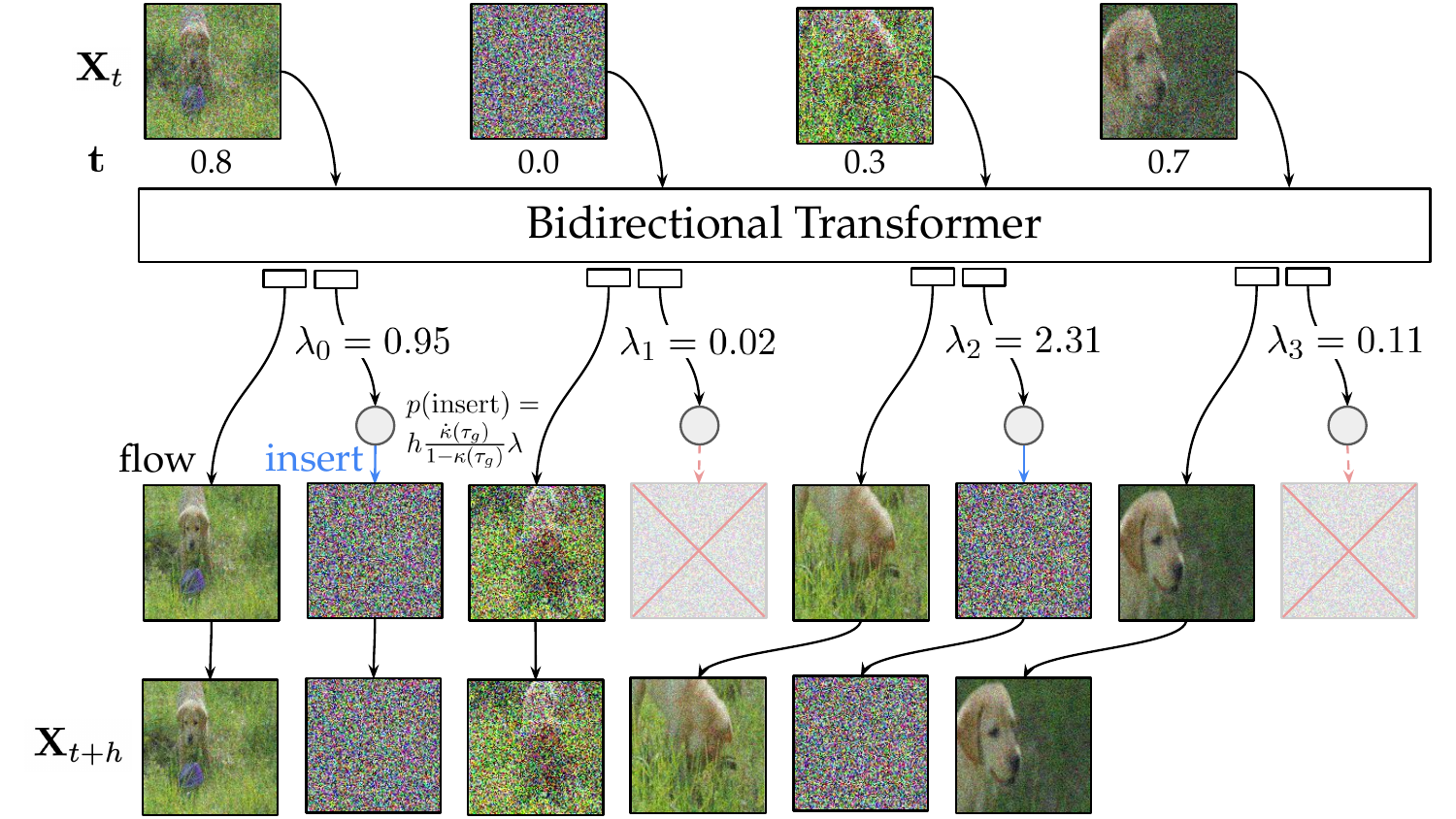}
    \caption{{\bf \ours sampling.} At each iteration, the model predicts, for each frame $i$, a  velocity field and an insertion rate $\lambda_i$. 
    Velocities are used to denoise frames while the insertion rates define the probability to  insert a new frame to the right of existing ones. 
    The model uses per-frame time values, set to  $t_i\!=\!0$ when they are inserted, and reaching $t_i\!=\!1$ when they are fully denoised.
    }
    \label{fig:flowception_sampling}
\end{figure}
    
In contrast, AR methods enable streaming generation by making each frame immutable once generated, allowing immediate display to users while subsequent frames attend to prior outputs. 
However, these AR approaches suffer from critical exposure bias \citep{zhang2025framepack,bengio15nips}: training uses ground-truth frames as context, while inference conditions on the model's own imperfect generations. 
This train-test mismatch prevents the model from learning to recover from its mistakes, causing minor artifacts to cascade as errors accumulate across frames, rapidly degrading video quality. 
Moreover, to enable KV caching, without which sampling from AR models becomes prohibitively expensive, AR methods typically use a causal attention mask which limits the expressiveness of the models. 
These trade-offs hinder long-term and efficient video synthesis.

In this work, we address the error accumulation and the limiting causal attention pattern of AR methods, while also avoiding the 
fixed-length generation requirement and reducing the computational cost of full-sequence models.
To this end, we introduce {\bf \ours}, a {variable-length}, {non-autoregressive video generation framework} that interleaves two processes throughout sampling: 
(i) continuous flow matching  denoising of  existing frames, and 
(ii) stochastic discrete  insertion of frames 
in between existing ones. 
Because the model flexibly determines the  insertion locations, the same model naturally performs text-to-video, image-to-video, video-to-video, video interpolation, and scene completion by toggling which conditioning input frames  are {\it active} (allow insertions to their right) or {\it passive} (no insertions). 
See    \Cref{fig:teaser} for an illustration of image-to-video and video interpolation results generated by our model.

Concretely, at each timestep the model predicts a velocity field over the existing frames and a per-frame insertion rate.
By sampling based on the insertion rate, a new frame may be inserted and initialized with a sample from a unit Gaussian  prior distribution. 
The new frame is then  subsequently denoised in the context of other already instantiated and partially denoised frames. 
This yields a coupled ODE–jump process over variable-length sequences and supports {\it any-order, any-length} generation by design.  
See \Cref{fig:flowception_sampling} for an illustration of the generation process. 

Beyond a more flexible generative model, our interleaved insertion-denoising schedule delivers an  efficiency gain, since   early in sampling only a small active subset of frames is actively denoised while unrevealed frames are marginalized out. 
Under a linear frame insertion schedule, the computational cost of the  quadratic attention term along the generation trajectory averages to a third of the attention cost incurred by a full-sequence model. 
Compared to AR with KV caching, \ours shows comparable sampling cost with more robustness under low NFEs.

We conduct extensive experiments with three datasets, Tai-Chi-HD, RealEstate10K and Kinetics 600, and consider (text/image)-to-video generation and video interpolation tasks. 
We find that, beyond being a more flexible framework, \ours leads to comparable or better performance than full-sequence and autoregressive paradigms under the same compute budget,
with consistent improvements in terms of FVD and VBench metrics.

\textbf{Contributions:}
\begin{enumerate}[leftmargin=15pt]
    \item We introduce \ours, a theoretically grounded video generation framework, which couples learned frame insertions with continuous flow matching in a unified model.
    \item We show how \ours can solve different tasks by conditioning on any set of frames, based on their relative order alone.
    \item We present an efficiency analysis showing an average of $3\times$ FLOPs reduction during training, and $1.5\times$ during sampling, \wrt full-sequence model.
    \item We present extensive experimental results improving over full-sequence and autoregressive approaches on multiple datasets.
\end{enumerate}

%% file: sections/related_work.tex
\section{Related work}

\mypar{Diffusion \& Flows}
Diffusion~\citep{pmlr-v37-sohl-dickstein15,ho2020denoising, song2021scorebased, dhariwal2021diffusion} and flow models~\citep{lipman2023flow, ma2024sitexploringflowdiffusionbased, sd3} are a class of generative models that operate by learning to model the infinitesimal transitions of the probability path going from source (noise distribution) to target (data manifold). 
They are considered state-of-the-art generative modeling tools for modalities ranging from images~\citep{ldm, podell2024sdxl, chen2023pixartalpha, sd3}, audio~\citep{NEURIPS2023_e1b619a9,controllable-music}, and video~\citep{jin2024pyramidal,polyak2025moviegencastmedia}. 
These models generate data by starting with a simple prior such as a unit Gaussian, and iteratively refining  samples  to recover samples from the data distribution model.
The simple training recipe underlying these models proved very effective through  scalability and stable training~\citep{improved_conditions}.

\mypar{Video generation}
The success achieved by scaling text-to-image models \citep{podell2024sdxl, sd3, flux2024} paved the way towards applying the same paradigms to video generation.
Such works fall broadly into two categories.
The first line of work~\citep{opensora, opensora2, wan2025} follows the full-sequence paradigm where the video is treated as one large tensor and all frames are denoised simultaneously.
However, this results in prohibitively large attention computations, hindering the development of such methods for long sequences.

Another line of work \citep{deng2024nova, chen2025diffusion} takes an autoregressive approach to iteratively predict the next frame (or block of frames) conditioned on previous frames.
This approach is known to have difficulty generating long videos, as sampling fixes early frames once they are generated and later frames are conditioned on (imperfectly) generated earlier frames (unlike in training, where conditioning is always based on ground-truth frames), leading to error accumulation that harms long-term generation.
To resolve this, some methods rely on partial noising of the context~\citep{song2025historyguidedvideodiffusion}, but this prevents KV caching, making sampling prohibitively expensive.
Even without KV caching, our \ours remains significantly faster than these autoregressive methods due to its interleaved insertion and denoising mechanism, whereas autoregressive generation is fundamentally constrained by sequential frame-by-frame generation.

Another family of approaches factorizes long video generation into sparse keyframe synthesis followed by video interpolation.
MovieDreamer~\citep{zhao2024moviedreamer} first autoregressively predicts a sequence of visual keyframes and then applies an image-to-video model to synthesize short clips around each keyframe, using CLIP features of the anchor frame rather than the last generated frame to mitigate drift across clips.
ART-V~\citep{artv} instead performs frame-wise autoregressive generation in a text-image-to-video (TI2V) setting: each new frame is produced by a diffusion model conditioned on both the input keyframe and the previously generated frames, with a masked diffusion mechanism that selectively reuses information from the given frame to reduce appearance drift over long horizons.

Several other directions have been explored to reduce complexity.
These include the use of highly compressive tokenizers~\citep{HaCohen2024LTXVideo,nvidia2025cosmos}, hierarchical or packed representations that compress remote frames before attending to them~\citep{jin2024pyramidal, zhang2025framepack}, and architectural changes that distribute or route attention over multiple context sources~\citep{fan2025vchitect, cai2025moc, song2025historyguidedvideodiffusion, zhao2024moviedreamer}.
However, these approaches still depend on design choices about which context sources or time ranges to emphasize, rather than fully learning relevance end-to-end.

More related to our work, several works explored an autoregressive approach, but denoise multiple frames simultaneously, with later frames being noisier than earlier frames~\citep{ai2025magi1,xie2025pavdm}. Differently from \ours,  however, these approaches  maintain  a strict left-to-right frame insertion pattern, preventing coordination between remote frames early in the denoising process as \ours allows it.

To enable multiple use cases, such as image-to-video and video-to-video,  some works~\citep{wan2025, HaCohen2024LTXVideo}  concatenate a frame-wise mask to the input of the model, which allows to condition on clean frames at various positions in the video. 
Similarly, \ours is capable of performing various tasks by conditioning on multiple input frames, and optionally constraining insertions to be to the right of some of those.

%% file: sections/setup.tex

%% file: sections/method2.tex
\section{Flowception}
We propose  \emph{Flowception}, 
a natural combination of the continuous Flow Matching~\citep{lipman2023flow} and the discrete Edit Flow~\citep{havasi2025edit} frameworks, with the goal of flexible video generation where existing frames are simultaneously denoised and new frames are inserted during the generation process. 
We first describe the parameterization and the generation procedure for this model, then derive a training scheme that adheres to the generation procedure. We will present notation loosely with high-level intuitive explanations; full derivations are in the supplementary material.

\subsection{Setup and model parameterization}

Flowception operates in the space of variable-length sequences of frames $\mathcal{X}=\bigcup_{n=0}^{\infty} \mathbb{R}^{n\times H \times W\times C}$, where $n$ denotes the  length of each sequence, and $H, W, C$ are the fixed height, width, and channel dimensions for each frame. 
We denote $\ell(X)$ as the length of a sequence $X \in \mathcal{X}$.

During generation, \ours  transports sequences and time values $(X, t)$ from an initial noise distribution $p_\text{src}$ to a data distribution $p_\text{data}$ through a \textit{continuous-time framework}.
We handle variable length by assigning \textit{per-frame time values}, collected in a vector $t \in \mathcal{T} = \bigcup_{n=0}^{\infty} [0,1]^n$, associating each frame to a different noise level.
Hence the \textit{input} of a Flowception model is a sequence of frames $X \in \mathcal{X}$ and a vector of per-frame time values $t \in \mathcal{T}$, where $X$ and $t$ are the same length.

To incorporate sequence length changes, we make use of an \textit{insertion operation}~\citep{havasi2025edit}. Given a sequence $X \in \mathcal{X}$, an insertion index $i$, and a noise frame $\varepsilon \sim \mathcal{N}(0, I)$, we define
\begin{equation}
    \ins(X, i, \varepsilon) = (X^1,\dots, X^i, \varepsilon, X^{i+1}, \dots, X^{n}),
\end{equation}
where the $X^i \in \mathbb{R}^{H \times W\times C}$ are frames in the sequence. The insertion operation inserts a 0-SNR frame which will then be denoised following the flow matching framework.

Frame insertions and flows form the primitive operations that we use during generation. 
Hence the \textit{output} of a Flowception model with parameters $\theta$ has two components at every position $i \in \{1,\dots,\ell(X)\}$: 
\begin{enumerate}
    \item An \textbf{insertion rate} $\lambda_i^\theta(X, t) \in \mathbb{R}_{\geq 0}$, which predicts the number of missing frames.
    \item A \textbf{velocity} $v_i^\theta(X, t) \in \mathbb{R}^{H \times W\times C}$, for denoising the existing frames.
\end{enumerate} 
These two components allow the model to simultaneously denoise an existing sequence of frames  and interpolate or extend the existing frame sequence.

\subsection{Generation procedure}

We start the generation process by sampling a fixed number of \textit{starting frames} $n_\text{start}$,
\begin{equation}
    p_\text{src}(X) = \textstyle \prod_{i=1}^{n_\text{start}} \mathcal{N}(X^i; 0, I),
\end{equation}
and assigning each frame the time $t^i=0$. 
We use a non-zero number of starting frames to start the generation process, otherwise the initial generation steps will simply be inserting noise frames anyway.

We then iteratively apply a transport step until all frames are clean, \ie, all $t^i=1$. 
Each transport step simultaneously flows the existing frames while also potentially inserting new frames with probability given by the insertion rates. 

\mypar{Global time value} The stopping criterion for the flow is when $t^i = 1$ for all frames, \ie we no longer modify any frame.
To impose a stopping criterion for the insertions, we introduce a \textit{global time value} $t_g$ which starts at $t_g = 0$ and frame insertions are only allowed while $t_g < 1$. Implementation-wise, we don't feed $t_g$ explicitly into the model, and is only tracked for sampling. This closely follows the design choice of OneFlow~\citep{nguyen2025oneflowconcurrentmixedmodalinterleaved} where insertion rates are not conditioned on time.
Accordingly, with a constant step size during inference, the maximum number of iterations is twice the amount needed for $t_g$ to arrive to $1$.

\mypar{Scheduler} 
We impose a distribution on the fraction of visible frames based on a monotonic scheduler $\kappa(t_g)$ that is a function of the global time value, with $\kappa(0)=0$, $\kappa(1)=1$.
Given a data sequence $X_1 \sim p_\text{data}$, the probability of each non-starting frame being in the sequence at a global time $t_g$ is, for each frame $X^i \in X_1$,
\begin{equation}\label{eq:schedule}
    \mathbb{P}(X^i \text{ in } X \;|\; \text{ global time }t_g ) = \kappa(t_g).
\end{equation}
We will primarily work with the linear scheduler $\kappa(t) = t$.

\mypar{Transport step} 
Given a sequence $X$ with per-frame times $t$, each transport step performs the following two operation simultaneously at all positions $i \in \{1, \dots, \ell(X)\}$:
\begin{enumerate}
    \item Apply the flow step
    \begin{equation}
        X^i = X^i + h v_i^\theta(X, t).
    \end{equation}
    \item With probability $h \tfrac{\dot{\kappa}(t_g)}{1 - \kappa(t_g)} \lambda_i^\theta$, insert a new frame
    \begin{equation}
        X = \ins(X, i, \varepsilon), \quad t = \ins(t, i, 0).
    \end{equation}
\end{enumerate}
The ratio $\tfrac{\dot{\kappa}(t_g)}{1 - \kappa(t_g)}$ ensures insertions occur in alignment with the distribution of visible frames imposed by the scheduler.

To summarize, we start with a fixed number of $n_\textrm{start}$ starting frames initialized as noise, advance a global time value $t_g$ from $0$ to $1$, and perform frame insertions along the way following the scheduler $\kappa(\cdot)$ in \Cref{eq:schedule}. 
Noise frames are inserted and then subsequently denoised, naturally producing per-frame time values $t_i$ and with a delay in the time values of later-inserted frames (\ie, $t_i \leq t_g$). We stop the generation when all frames have reached $t_i = 1$.
An illustration of a transport step is shown in \Cref{fig:flowception_sampling}. A pseudocode implementation is in the supplementary material. 

\subsection{Training procedure}

The \ours generation procedure induces a particular distribution over the visible frames and their time values. 
Frames may be at different time values, and some may not  be inserted yet. 
We must align with this distribution of time values and missing frames during training so that there is no distribution mismatch between training and generation. 
Despite this complexity, the distribution is controlled completely by the scheduler $\kappa$. In this section, we describe a simple training recipe to easily sample all time values and the states of each frame during training. 

\mypar{Extended time values} 
Extended time values are denoted $\tau$ and can take values outside of the interval $[0, 1]$, so we define a clip operation that results in real time values,
\begin{equation}
    t = \text{clip}(\tau) := \max\{0, \min\{1, \tau\} \}.
\end{equation}
Firstly, after the global time value reaches $t_g = 1$, while no insertion of new frames can occur, existing frames may still need more steps to be fully denoised. Thus we introduce an \emph{extended global time} $\tau_g \in [0, 2]$, with $t_g = \text{clip}(\tau_g)$.

Secondly, while $\tau_g \leq 1$, insertions occur according to the scheduler $\kappa(\cdot)$. 
For the linear scheduler new frames are inserted uniformly across the time interval $t_g \in [0, 1]$.
For the general case, please see the detailed derivations in the supplementary material.
Therefore, the time delay between the global extended time and per-frame extended time values follows:
\begin{equation}\label{eq:time_delay}
    u_i = \tau_g - \tau_i \sim \text{Unif}(0, 1), \quad \text{ and so } \quad \tau_i = \tau_g - u_i, 
\end{equation}
where $t_i = \text{clip}(\tau_i)$ computes the real per-frame times.
Note that $\tau_i \in [-1, 2]$ according to \eqref{eq:time_delay}. When $\tau_i < 0$, the frame is yet to be inserted and thus in a ``deleted'' state. 

\Cref{fig:flowception_time} illustrates the different frame states according to the extended time schedule. 
Each frame can either 
\begin{itemize}
    \item be in \emph{deleted state}  ($\tau_i < 0$),
    \item exist and be in \emph{flow state}  ($0 \leq \tau_i < 1$) , or
    \item be frozen and in \emph{terminal state} ($\tau_i \geq 1$).
\end{itemize}

During training, we  sample noisy sequences using
\begin{align}
    \tau_g \sim p(\tau_g), \quad \tau_i = \tau_g - u_i, \\
    t_i = \text{clip}(\tau_i), \quad X_1^\text{visible} = ( X^i \in X_1 | \tau_i > 0 ), \\
    X = t X_1^\text{visible} + (1-t) X_0, \quad X_0 \sim \mathcal{N}(0, I),
\end{align}
where $X_1 \sim p_\text{data}$ is a target sequence sample.

\begin{figure}
    \centering
    \includegraphics[width=.5\textwidth]{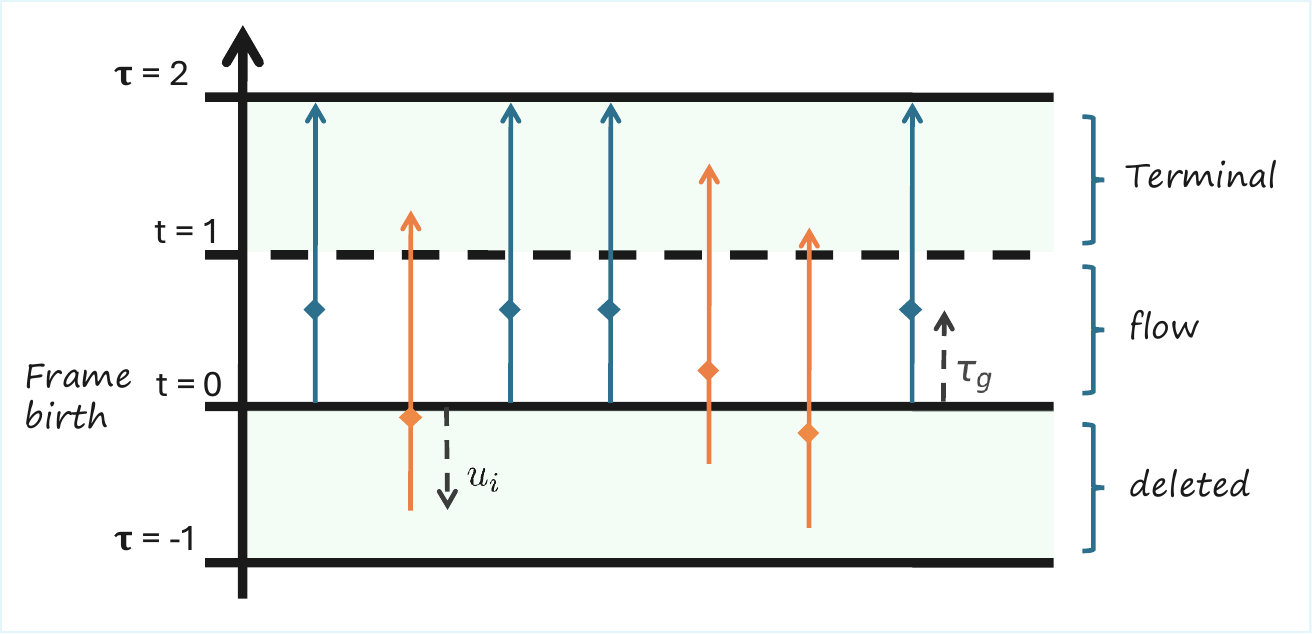}
    \caption{{\bf Illustration of the extended time scheduler for \ours training.} 
    In \ours, each frame has its own denoising time which depends on its insertion time.
    The global extended time $\tau_g$ progresses from $0$ to $2$, where insertion of new frames only occur when $\tau_g < 1$.
    Starting frames (in blue) are instantiated  at $\tau_g=0$, other frames (in orange) are  inserted later (when $\tau_g > 0$) and thus have a delay.
    With a linear scheduler, the insertion delays  follow a uniform distribution. 
    }
    \label{fig:flowception_time}
\end{figure}

\mypar{Insertion loss} 
At each position $i$ of the noisy sequence $X$, we denote $k^i$ as the number of missing frames to the right, between positions $i$ and $i+1$. Following OneFlow~\citep{nguyen2025oneflowconcurrentmixedmodalinterleaved}, the loss for training the insertion output $\lambda_i$, which predicts the number of missing frames, is given by the negative log-likelihood of the Poisson distribution
\begin{equation}
\mathcal{L}_{\mathrm{ins}} =
\sum_{i=1}^{\ell(X)} \lambda_i^\theta(X, t)\;-\;k^i\,\log\lambda_i^\theta(X, t).
\label{eq:lins-counts}
\end{equation}

\mypar{Velocity loss} At each position $i$ of the noisy sequence $X$, we learn to denoise the frame with the standard Flow Matching loss \citep{lipman2023flow},
\begin{equation}
    \mathcal{L}_\text{vel} = \left\lVert v^\theta(X, t) - (X_1^\text{visible} - X_0) \right\rVert^2.
\end{equation}

\subsection{Practical considerations \& implications}
\mypar{Architecture with per-frame time conditioning}
Differently from other frameworks, our model has two prediction heads: the dense velocity prediction head and the insertion rate prediction head.
In order to enable different timesteps per frame, we change the AdaLN~\citep{Peebles2022DiT} modulation to operate on a per-frame basis.
To distinguish between noisy and input/conditioning frames, we inflate the channel dimension of the inputs so that the first $C$ channels contain the noisy frames while the second group contains the input/conditioning frames (or zero padding).
Finally, for the rate prediction head, a learnable token is concatenated to each frame's tokens before being projected with a simple MLP and an exponential activation.

\begin{figure}[t]
    \centering
    \includegraphics[width=.5\textwidth]{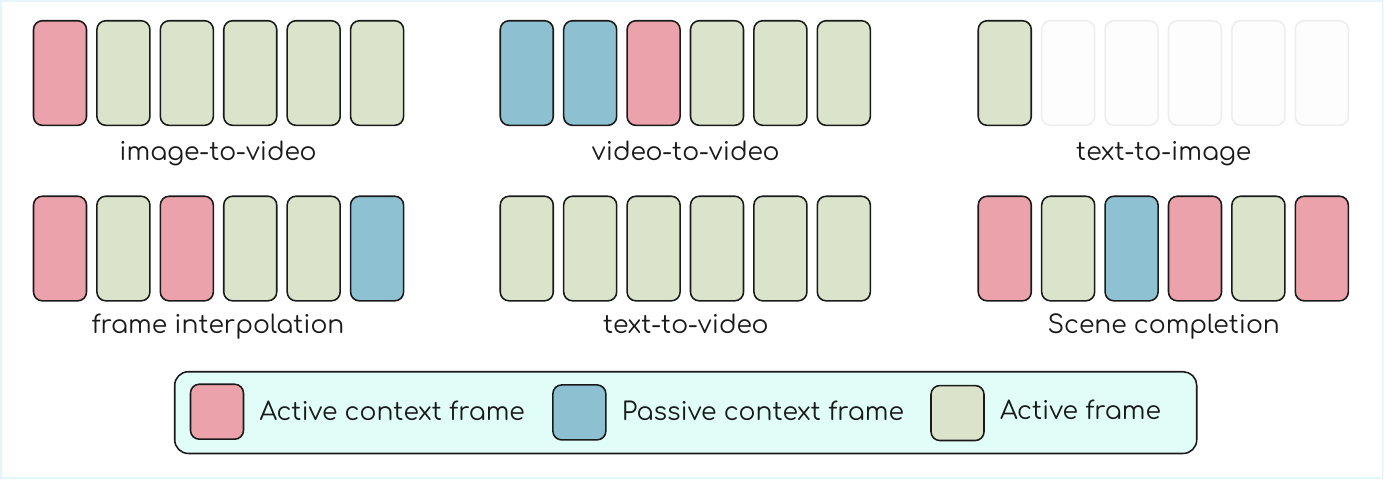}
    \caption{{\bf \ours natively supports different tasks.} 
    By choosing active and passive context frames the model can be used for 
    image-to-video,   video-to-video, text-to-image,
    frame interpolation, text-to-video,  and scene completion.
    }
    \label{fig:flowception_setups}
\end{figure}

\mypar{Supporting different tasks}
Since the frames present at time $t$ are only specified by their relative order, \ours  supports multiple tasks by feeding different {\it context frames}.
We distinguish two sorts of context frames: {\it active} ones which can induce insertions, and {\it passive} ones which do not.
As illustrated in \Cref{fig:flowception_setups}, mixing these context frames in different ways allows for image-to-video, video-to-video interpolation and scene completion without needing to specify the size of the gaps but only the relative order of the context frames.

\mypar{Computational cost} 
Early in sampling, only a small subset of frames is visible, which reduces the computational cost of these denoising steps compared to full-sequence denoising. 
Under linear $\kappa$ we have  $\mathbb{E}_{p(\tau_g)} [\kappa(\tau_g)^2]=\int_0^1 \tau_g^2 d\tau_g=1/3$, and so the expected attention FLOPs are $1/3$ of full-sequence Flow Matching, assuming an  equal number of denoising steps used in both models. 
During sampling, as the latest flow update occurs at $\tau_g=2$, our model uses  at most twice the flow iterations of the full-sequence model assuming the same time discretization, \ie $2/3$ the FLOP count of full-sequence Flow Matching.
Compared to AR models, Flowception has a comparable sampling cost with more robustness under low NFEs.
A more detailed analysis is available in the supplementary material.

%% file: sections/experiments.tex
\section{Experiments}

\input{tables/i2v_results_2}

\begin{figure*}
    \centering
    \includegraphics[width=\linewidth]{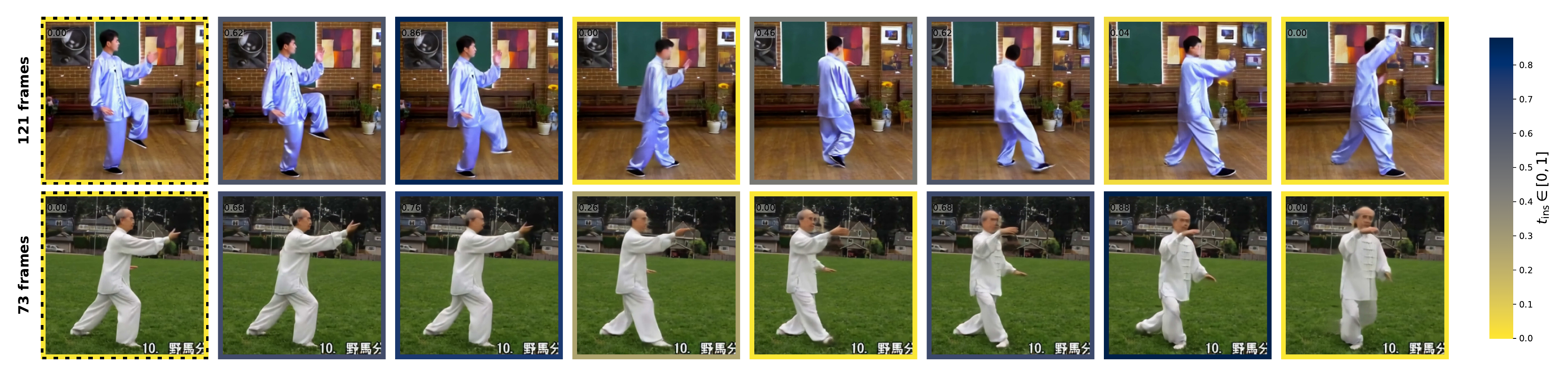}
    \caption{{\bf Image-to-video generation using model trained on Tai-Chi-HD.} We show two examples (one per row) of  generated video frames along the insertion time of each frame.
    In general,  frames inserted early, with $t\approx0$, define the movement dynamics (large changes \wrt the context frame), while later frames smoothly interpolate motion, resulting in smaller changes \wrt neighboring frames.
    }
    \label{fig:taichi_tins}
\end{figure*}

\subsection{Experimental setup}
\label{sec:setup}

\mypar{Model architecture}
Our architecture builds on the popular DiT diffusion transformer architecture~\citep{Peebles2022DiT}, with minimal changes to enable insertion rate predictions.
We equip  each attention block  with an additional learnable token which is replicated for each frame and takes part in the attention computation, ultimately each of these tokens maps to the non-negative rate prediction using an exponential activation.
Our baseline model is made of 38 attention blocks with a hidden dimension of 1536, each attention head has a dimension of 64, resulting in 24 total heads, for a model with approximately 2.1B learnable parameters.
Positional embeddings follow the VideoROPE method~\citep{wei2025videorope}.
We enable full bidirectional attention between frames and optionally concatenated text tokens in an MMDiT~\citep{sd3} fashion.
For T2V results, we finetune an open-source LTX-2b model (0.9.5)~\citep{HaCohen2024LTXVideo} with our framework after conducting the architecture changes necessary for \ours.

\mypar{Datasets}
We evaluate our model on three different datasets.
We used two narrow-domain datasets:  Tai-Chi-HD~\citep{siarohin2019firstordermotionmodel} and RealEstate10K~\citep{zhou2018stereo}, as well as the class-structured Kinetics 600 dataset~\citep{carreira2018kinetics600}. 
Unless specified otherwise, we train our models at $256$ resolution and generate up to 145 frames at 16 FPS.
Furthermore, we trained a general domain model using a proprietary dataset of around 2M videos ranging approximately from 20 to 240 seconds.
We show results for this model in the supplementary material.
We use the LTX autoencoder~\citep{HaCohen2024LTXVideo}, which has a spatial downsampling factor of 32  and temporal downsampling  factor of 8.
To correctly support variable length video decoding, we modify the decoder architecture to propagate a frame validity mask which is used to mask out invalid frames in the hidden layers, preventing border effects where padding frames leak into the video. 
When prompting models with one or more video frames to generate a video, we  choose these reference frames from the validation set.

\mypar{Metrics}
To evaluate the performance of our models, we rely on the standard FVD  metric~\citep{unterthiner2019fvd}, computed with respect to a subset of 5k videos from the training set, as well as selected metrics from  VBench~\citep{huang2024vbench}.
Specifically, we report imaging and aesthetic quality, background and subject consistency, motion smoothness and dynamic degree.

\mypar{Baselines}
We compare our method with popular video generation paradigms including 
(i) \emph{full-sequence generation} where all frames are flowed simultaneously, sharing a global timestep and  bidirectional attention over all frames, and 
(ii) \emph{per-frame autoregressive generation} where  the model iteratively predicts the next frame(s) conditioned on the previously available ones.  
For autoregressive generation, we experiment with both  bidirectional and causal attention (which allows  KV caching for efficient inference).

We sample videos from \ours and baselines without using guidance, unless specified otherwise. 
The complete videos from which we show selected frames in \Cref{fig:teaser,fig:taichi_tins,fig:kinetics_interp} are included in the supplementary material to better appreciate the video quality.

\subsection{Main results}
\label{sec:results}

\mypar{Image-to-Video (I2V) generation}
In \Cref{tab:i2v_results} we compare our \ours approach to the  full-sequence and autoregressive baselines for the I2V task where we condition on the first frame  to generate the rest of the video. 
We set the generated video length to 145 frames.
Overall, we find the autoregressive and full-sequence model to give fairly similar results. 
For the VBench metrics, \ours leads to better results in most cases, or second best otherwise. 
For video quality as measured in FVD, \ours improves results across the board.

In \Cref{fig:taichi_tins} we show two qualitative examples of I2V generation, which demonstrate the ability of the model to generate smooth video sequences with substantial motion and appearance consistency. 
Additionally, we observe an emergent ``coarse-to-fine'' structure in the order of generated frames:  early on in the generation process far-away frames are generated that tend to vary and define the overall motion of the video, while later frames mostly interpolate between them.

\mypar{Text-to-video (T2V) generation}
We train \ours on the proprietary dataset of around 2M text-video pairs described in \Cref{sec:setup}.
We use a pretrained LTX model for this experiment as stated earlier.
Our results are reported in \Cref{tab:t2v_results}, \ours consistently improves perceptual quality and temporal coherence over the full-sequence paradigm, achieving higher imaging quality, background consistency, aesthetics, motion smoothness and subject consistency while dynamic degree is only slightly deteriorated.
Qualitatively, \ours produces sharper motion and fewer identity drifts over long time horizons (Fig.~\ref{fig:teaser}).
We empirically observe similar trends to I2V where early frames define coarse motion while later inserted frames ensure smooth transitions between key frames.

\begin{table}[h]
    \centering
    \resizebox{\linewidth}{!}{
    \begin{tabular}{lcccccc}
        \toprule
         {\it CFG=7.0, 50 NFEs} & Imag.$\uparrow$ & Back.$\uparrow$ & AES$\uparrow$ & Motion$\uparrow$ & Subj.$\uparrow$ & Dyna.$\uparrow$\\
         \midrule
        LTX (original) & 49.22 & 95.13 & 48.25 & 98.49 & 92.30 & 53.90\\
        \midrule
        LTX (full-sequence) & 47.96 & 93.82 & 47.28 &98.07 & 89.47 & \cellhi 53.71\\
        LTX (\ours) & \cellhi 51.37 & \cellhi 94.95 & \cellhi 49.56 &\cellhi  98.14 & \cellhi 91.40 & 52.43\\
        \bottomrule
    \end{tabular}}
    \caption{{\bf Text-to-video generation results.} We adapt LTX-2b model and compare \ours with full-sequence training.}
    \label{tab:t2v_results}
\end{table}

\begin{figure*}
    \centering
    \includegraphics[width=\linewidth]{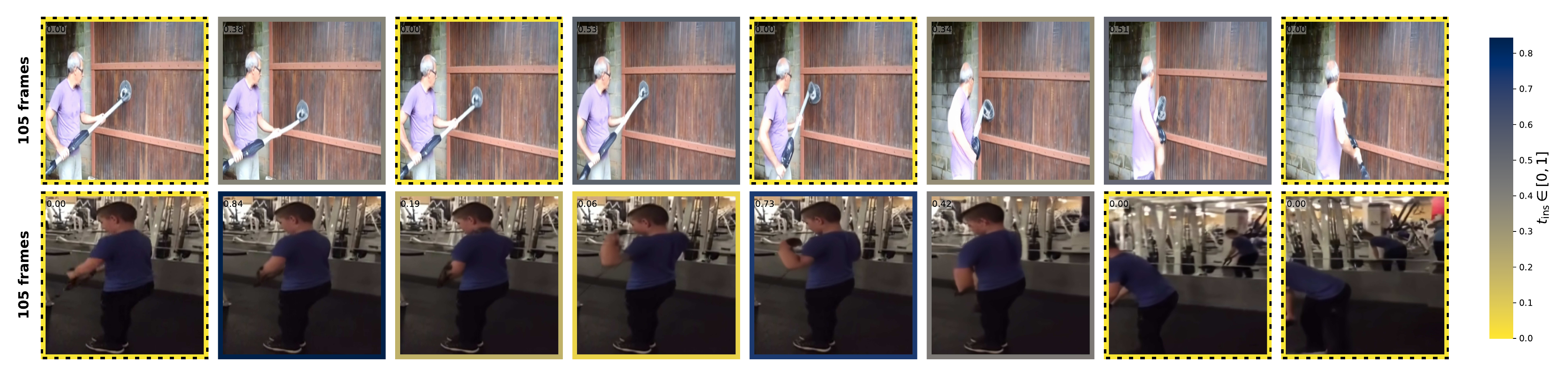}
    \caption{{\bf Frame interpolation results obtained with \ours trained on the Kinetics-600 dataset.}
    Context frames are highlighted with dashed lines, insertion times of other frames are marked using  a color map and printed on each frame. 
    }
    \label{fig:kinetics_interp}
\end{figure*}

\mypar{Frame interpolation}
Another interesting task directly supported by \ours is frame interpolation, where rather than a single starting frame, we provide an ordered  set of context frames as conditioning, and let the model  interpolate the video freely between these. 
Differently from other frameworks \citep{wan2025, HaCohen2024LTXVideo}, \ours  does not need to be given  the number of frames between successive conditioning frames, but it rather  inserts frames as it sees fit, resulting in a more flexible interpolation capability.
We illustrate this use-case in \Cref{fig:kinetics_interp}.
We observe that the model is able to adapt the number of insertions to generate coherent scenes with smooth transitions. For instance, in the second row, no frame is inserted after the penultimate context frame in order to preserve motion continuity.

\begin{table}
    \caption{{\bf Comparing insertion schemes.} Using the \ours rate prediction  and baselines that insert frames in a random order,  in a hierarchical scheme (see text for detail), and left-to-right. Models trained on RealEstate10K.
    }
    \centering
    {\scriptsize 
    \begin{tabular}{lccc}
        \toprule
        Insertion method & FVD & Motion & Dynamic\\
        \midrule
        Random          & 25.03 & 99.09 & 70.68\\
        Hierarchical   & 23.94 & 99.10 & 71.20\\
        Left-to-right  & 23.61 & 99.03 & 73.04\\
        \ours rate prediction & \cellhi 21.80 & \cellhi 99.30 & \cellhi 78.59\\
        \bottomrule
    \end{tabular}}
    \label{tab:insertion_rule}
\end{table}

\mypar{Sampling Efficiency}
We compare wall-clock  time for  sampling the same LTX-2b model used for our T2V experiment. 
We track sampling time on a single H200 GPU with different numbers of sampling steps (NFEs). \ours is consistently about 30\%  faster than the Full-Sequence baseline.

\begin{figure}[h]
    \centering
    \includegraphics[width=0.75\linewidth]{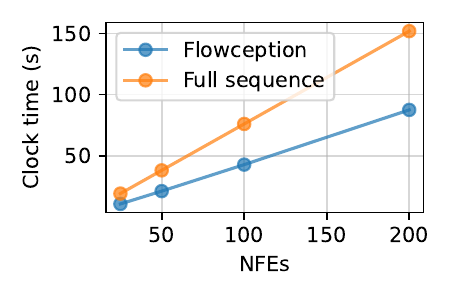}
    \caption{{\bf Sampling speed.} We compare wall-clock time when sampling videos using either full-sequence or Flowception paradigms.}
    \label{fig:placeholder}
\end{figure}

\subsection{Ablation experiments}
\label{sec:ablations}

\mypar{Impact of learned insertions}
To assess the importance of the insertion head that learns to predict where to insert frames, we compare it to inserting frames in a data-independent manner.
In particular, we first determine the video length $n$. 
We  then consider inserting frames in a random order, using a ``hierarchical'' scheme where we iteratively insert them in the middle of the largest interval in $\{1,\dots,n\}$ where no frame has been inserted yet, or using a left-to-right pattern  as used in AR models.
We present the results of this experiment in \Cref{tab:insertion_rule}.
We find that across metrics, results improve as we move from the random pattern to the hierarchical pattern, to the left-to-right, and finally to the data-driven insertion patterns in \ours. 

Note that for the left-to-right pattern, the FVD results in \Cref{tab:insertion_rule} (23.61) are much better than the FVD for the AR baselines in \Cref{tab:autoregressive_comp} (47.48 and 45.13).
We hypothesize that this difference is due to the fact that the AR baselines  denoise the next frame once previous ones are already fully denoised and committed to, while in \ours new frames start denoising while previous ones are still not fully denoised (more similar to the full-sequence baseline which yields a FVD of 26.17 on RealEstate10K).

\mypar{Insertion guidance}
Classifier-free guidance (CFG) is commonly used to achieve better prompt alignment and image quality in diffusion and flow models~\citep{ho2021classifierfree}.
Similarly, in \ours,  we can perform CFG on the insertion rates. 
Following CFG,  we randomly  drop the conditional information $c$ during training, and define the guided update as
\begin{equation}
    \lambda^\text{cfg}(X_t | c) = \lambda(X_t | c)^{w_s} \lambda(X_t)^{1-{w_s}},
\end{equation}
where $w_s\geq1$ is the guidance scale for the insertion rate.
\Cref{fig:sinsert_impact} presents a scatter plot comparing video lengths for I2V obtained with and without guidance, initiating generation from the same seed for each  starting frame. For efficiency we limit the number of frames to 20 in this experiment. 
The result clearly indicates that using rate guidance ($w_s\!=\!5$) biases the model towards generating longer videos.

\ifcvpr
{
    \begin{figure}
        \centering
        \includegraphics[width=.6\linewidth]{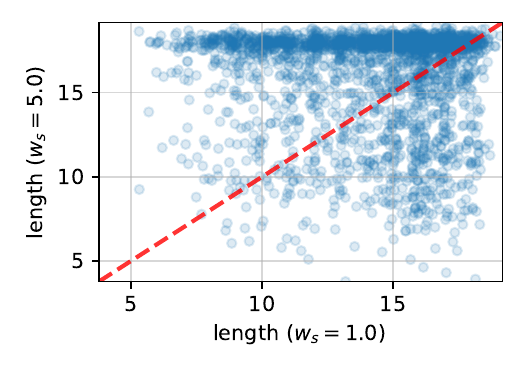}
        \caption{{\bf Impact of rate guidance on I2V video length.} Samples obtained  without guidance ($w_s\!=\!1$, horizontal) and with guidance ($w_s\!=\!5$, vertical) on RealEstate10K, using the same seed for each conditioning frame.
        }
        \label{fig:sinsert_impact}
    \end{figure}
}
\else
{
    \begin{figure}
        \centering
        \includegraphics[width=.65\linewidth]{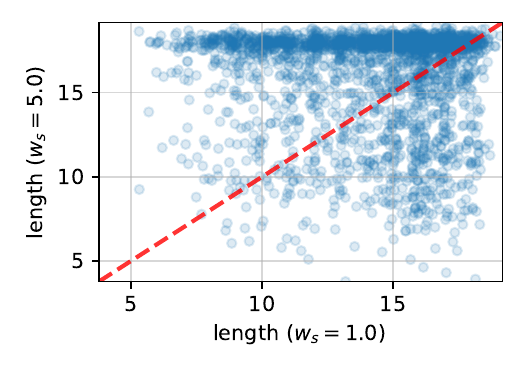}
        \caption{{\bf Impact of rate guidance on I2V video length.} Samples obtained  without guidance ($w_s\!=\!1$, horizontal) and with guidance ($w_s\!=\!5$, vertical) on RealEstate10K, using the same seed for each conditioning frame.
        }
        \label{fig:sinsert_impact}
    \end{figure}
}
\fi

Without guidance, or using lower values,  we find insertion guidance to prevent problems of under-insertions which can result in a choppy transition between two frames.
This is confirmed in \Cref{tab:rate_cfg_smoothness} where motion smoothness increases while dynamic degree decreases as $w_s$ increases.
\begin{table}[h]
    \caption{{\bf Effect of rate guidance on motion smoothness.} We report dynamic degree and motion smoothness for various values of the insertion rate guidance parameter $w_s$.
    }
    \centering
    \scriptsize{
    \begin{tabular}{lccc}
        \toprule
        $w_s$ & FVD & Motion & Dynamic\\
        \midrule
        1.0 & \cellhi 21.80 & 99.30 & 78.59\\
        2.0 & 22.69 & 99.31 & \cellhi 78.61\\
        5.0 & 25.30 & \cellhi 99.33 & 77.78\\
        \bottomrule
    \end{tabular}}
    \label{tab:rate_cfg_smoothness}
\end{table}

\mypar{Causality in AR models}
In the  literature, most autoregressive video generation models make use of frame-wise causal attention~\citep{deng2024nova, chen2025diffusion}, which allows for KV caching  and therefore efficient inference for autoregressive generation.
Other works improve on quality and long term rollout consistency by using bidirectional attention in the context and applying frame-wise noising for each frame~\citep{song2025historyguidedvideodiffusion}.
Such an approach, while effective, significantly raises inference time expected compute from $O(L^2n^2)$ to $O(L^2 n^3)$ where $n$ is the number of video frames, and $L$ is the number of tokens per frame. 
A more detailed analysis of FLOPs is provided in the supplementary material.

We provide a comparison between both autoregressive methods in \Cref{tab:autoregressive_comp}. 
The results demonstrate slightly better performance of the non-causal attention pattern, outperforming the causal one on every metric.  
\ours yields further improvements over the non-causal AR baseline on all metrics but subject consistency. 

\begin{table}[h]
    \caption{{\bf Comparison of autoregressive attention patterns.} Autoregressive baseline with different attention masks and comparison to \ours in terms of VBench metrics and FVD. Models trained on RealEstate10K.}
    \centering
    \resizebox{\linewidth}{!}{
    \begin{tabular}{lccc}
    \toprule
    & AR causal & AR non-causal & \ours \\
    \midrule
    FVD & 47.48 & 45.13 & \cellhi 21.80\\
    Imaging quality & 48.55 & 48.70 & \cellhi 51.18\\
    Background consistency & 93.84 & 93.88 & \cellhi 96.93\\
    Aesthetic  quality & 44.48 & 45.28 & \cellhi 48.09\\
    Motion smoothness & 99.16 & 99.20 & \cellhi 99.30\\
    Subject consistency & 87.29 & \cellhi 87.46 & 87.02\\
    Dynamic degree& 72.60 & 73.52 & \cellhi 78.59\\
     \bottomrule
    \end{tabular}}
    \label{tab:autoregressive_comp}
\end{table}

\mypar{Local \vs global attention}
\ours generates a video by progressively inserting new frames.
Empirically, we observe that  \ours tends to insert distant frames first, and then progressively ``fills in the blanks''. 
This suggests that \ours could be more amenable to using efficient local attention patterns, as far-away frames can still attend to each other early in the denoising process when the sequence is still short. 
To assess whether \ours is more amenable to using local attention windows than the full-sequence baseline, we compare  models using global attention and using local attention windows restricted to the  previous and next $K$ frames in the sequence. 
The results in \Cref{fig:local_attn_fvd} indicate that \ours suffers far less from using local attention windows than the full-sequence model, presumably due the communication between distant frames early in the \ours denoising process, when the intermediate frames have not yet been inserted, even when using small attention windows.

\begin{figure}[t]
    \centering
    \includegraphics[width=.6\linewidth]{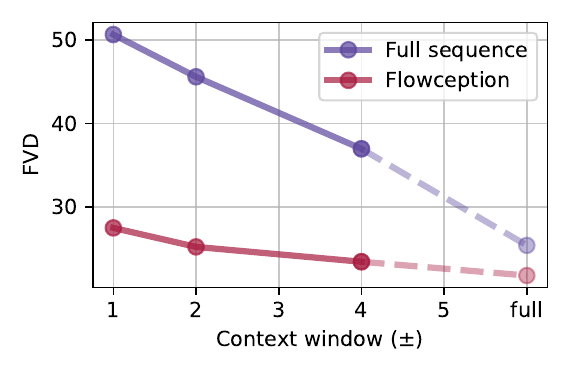}
    \caption{{\bf Comparing local attention variants.} We report FVD on RealEstate10K after training for 300k iterations with different attention context windows. For each frame, its context is made of itself and the previous/next context window frames.
    }
    \label{fig:local_attn_fvd}
\end{figure}

%% file: tables/i2v_results_2.tex
\ifcvpr
{
    \begin{SCtable*}
    \caption{\textbf{Image-to-Video generation results.} Models trained and sampled at 256 resolution. Comparing our Flowception approach to the Full Sequence and (causal) Autoregressive baselines. VBench Quality metrics include imaging quality, background consistency, aesthetic quality, motion smoothness, subject consistency, and dynamic degree. Best results for each dataset are highlighted.
    }
    \centering
     \scriptsize
    \begin{tabular}{@{}l l cccccc c@{}}
    \toprule
    & & \multicolumn{6}{c}{\textbf{VBench  Metrics}} & \multicolumn{1}{c}{\textbf{Quality}} \\
    \cmidrule(lr){3-8} \cmidrule(lr){9-9}
    \textbf{Dataset} & \textbf{Method} & 
    \textbf{Imaging$\uparrow$} & 
    \textbf{Background$\uparrow$} & 
    \textbf{Aesthetic$\uparrow$} & 
    \textbf{Motion$\uparrow$} & 
    \textbf{Subject$\uparrow$} & 
    \textbf{Dynamic$\uparrow$} & 
    \textbf{FVD$\downarrow$} \\
    \midrule
    \multicolumn{9}{l}{\textit{Kinetics-600}} \\
    & Full Seq. & 37.09 & 94.75 & 39.42 & \cellhi 99.39 & 92.46 & 44.35 & 204.65 \\
    & Autoreg. & 38.77 & 92.69 & 38.17 & 98.11 & 85.82 & \cellhi 54.66 & 201.34 \\
    & Flowception & \cellhi 41.92 & \cellhi 96.96 & \cellhi 42.05 & 99.21 & \cellhi 94.74 & 47.07 & \cellhi 164.73 \\
    \midrule
    \multicolumn{9}{l}{\textit{Tai-Chi-HD}} \\
    & Full Seq. & 47.48 & 94.42 & 54.93 & 99.26 & 92.18 & 18.61 & 27.30 \\
    & Autoreg. & 47.15 &  95.93 & \cellhi 54.98 & 99.37 & 93.41 & \cellhi 21.23 & 25.30 \\
    & Flowception & \cellhi 48.42 & \cellhi 95.93 & 54.96 & \cellhi 99.67 & \cellhi 94.43 & 20.02 & \cellhi 25.21 \\
    \midrule
    \multicolumn{9}{l}{\textit{RealEstate10K}} \\
    & Full Seq. & 50.11 & 93.48 & 44.53 & 99.08 & 85.85 & \cellhi 81.64 & 26.17\\
    & Autoreg.  & 48.55 & 93.84 & 44.48 & 99.16 & \cellhi 87.29 & 72.60 & 47.48 \\
    & Flowception & \cellhi 51.18 & \cellhi 96.93 & \cellhi 48.09 & \cellhi 99.30 & 87.02 & 78.59 & \cellhi 21.80\\
    \bottomrule
    \end{tabular}
    \label{tab:i2v_results}
    \end{SCtable*}
}
\else
{
    \begin{SCtable*}
    \caption{\textbf{Image-to-Video generation results.} Models trained and sampled at 256 resolution. Comparing our Flowception approach to the Full Sequence and (causal) Autoregressive baselines. VBench Quality metrics include imaging quality, background consistency, aesthetic quality, motion smoothness, subject consistency, and dynamic degree. Best results for each dataset are highlighted.
    }
    \centering
     \resizebox{1.4\linewidth}{!}{
    \begin{tabular}{@{}l l cccccc c@{}}
    \toprule
    & & \multicolumn{6}{c}{\textbf{VBench  Metrics}} & \multicolumn{1}{c}{\textbf{Quality}} \\
    \cmidrule(lr){3-8} \cmidrule(lr){9-9}
    \textbf{Dataset} & \textbf{Method} & 
    \textbf{Imaging$\uparrow$} & 
    \textbf{Background$\uparrow$} & 
    \textbf{Aesthetic$\uparrow$} & 
    \textbf{Motion$\uparrow$} & 
    \textbf{Subject$\uparrow$} & 
    \textbf{Dynamic$\uparrow$} & 
    \textbf{FVD$\downarrow$} \\
    \midrule
    \multicolumn{9}{l}{\textit{Kinetics-600}} \\
    & Full Seq. & 37.09 & 94.75 & 39.42 & \cellhi 99.39 & 92.46 & 44.35 & 204.65 \\
    & Autoreg. & 38.77 & 92.69 & 38.17 & 98.11 & 85.82 & \cellhi 54.66 & 201.34 \\
    & Flowception & \cellhi 41.92 & \cellhi 96.96 & \cellhi 42.05 & 99.21 & \cellhi 94.74 & 47.07 & \cellhi 164.73 \\
    \midrule
    \multicolumn{9}{l}{\textit{Tai-Chi-HD}} \\
    & Full Seq. & 47.48 & 94.42 & 54.93 & 99.26 & 92.18 & 18.61 & 27.30 \\
    & Autoreg. & 47.15 &  95.93 & \cellhi 54.98 & 99.37 & 93.41 & \cellhi 21.23 & 25.30 \\
    & Flowception & \cellhi 48.42 & \cellhi 95.93 & 54.96 & \cellhi 99.67 & \cellhi 94.43 & 20.02 & \cellhi 25.21 \\
    \midrule
    \multicolumn{9}{l}{\textit{RealEstate10K}} \\
    & Full Seq. & 50.11 & 93.48 & 44.53 & 99.08 & 85.85 & \cellhi 81.64 & 26.17\\
    & Autoreg.  & 48.55 & 93.84 & 44.48 & 99.16 & \cellhi 87.29 & 72.60 & 47.48 \\
    & Flowception & \cellhi 51.18 & \cellhi 96.93 & \cellhi 48.09 & \cellhi 99.30 & 87.02 & 78.59 & \cellhi 21.80\\
    \bottomrule
    \end{tabular}}
    \label{tab:i2v_results}
    \end{SCtable*}
}
\fi

%% file: sections/conclusion.tex
\section{Conclusion}

We presented \ours, a novel video generation model that interleaves denoising existing frames with inserting new ones at appropriate locations, enabling efficient variable-length non-autoregressive generation.
Our method outperforms full sequence and autoregressive baselines in both generation quality and training efficiency. 
\ours naturally supports various prediction tasks including image-to-video, video-to-video generation, and frame interpolation.
Our work offers a promising alternative to standard autoregressive and full-sequence approaches for long-term generation and flexible editing.
While training on partial sequences allows for these emergent behaviors, it doubles the number of iterations in order to ensure complete denoising of all frames. Exploring different interleaved schedules for improved efficiency is an exciting future direction.

%% file: sections/appendix.tex
\section{Full derivations}
\label{app:flowception_full_derivation}

In this section, we provide the derivations for our model.
We start with a brief summarization of the Edit Flows framework in video space before deriving the interleaved time schedule for concurrent frame insertions and denoising and training losses.

\subsection{Edit flows and frame insertions}

\mypar{Setup}
As explained in the main manuscript, we model videos as sequences of frames from the space $\mathcal{X}$.
We use a blank token $\blank$ to mark empty positions in a sequence of videos.
Let $\mathcal{Z}=\bigcup_{n=0}^{N}(\mathcal{X}\cup\{\blank\})^n$ define the space were the (augmented) videos live and 
$\strip:\mathcal{Z}\to\mathcal{X}$ the mapping from augmented to observable space where $\strip$ removes all blanks (i.e $\X_t = \strip(\Z_t)$), and define the product delta on sequences
$\delta_{z_1}(z_2)=\prod_i\delta_{z_1^i}(z_2^i)$.

Under this parameterization, a sample $\Z_0 \in \mathcal{Z}$ in augmented space is a series of noise frames interleaved with blank tokens at random locations, as illustrated in \Cref{fig:flowception_coupling}.

\mypar{Conditional probability path}
We prescribe a coupling between source and target distributions.
We use the standard independent coupling where each clean frame is paired with an
independent Gaussian noise frame. 
Concretely, for the source we take
\[
  X_0 \sim \prod_{i=1}^k \mathcal{N}(X^i_0; 0, I),
\]
 and use an augmented variable $Z_t \in (\mathcal{X}\cup\{\blank\})^n$
to model masked insertions. At $t=0$, we start from the all-blank sequence
$Z_0 = (\blank, \dots, \blank)$ and gradually reveal the clean frames $X_1$
according to the scheduler $\kappa_t$. 

Given $\X_1\sim p_{\text{data}}$ a video with $n$ frames, we define a conditional masked path over $\Z_t\in(\mathcal{X}\cup\{\blank\})^n$ interpolating between $\X_0 \in \mathcal{N}(0,I)^k$ (with $k\leq n$) and $\X_1$ where transitions from blank frames to real frames follow the probability law:
\begin{align}
p_t(\X_t,\Z_t \mid \X_1)
  &= p_t(\X_t\mid \Z_t)\,p_t(\Z_t\mid \Z_1)
  \label{eq:flowception_cond_pt}\\
= \delta_{\strip(\Z_t)}(\X_t)\,
   & \prod_{i=1}^{n}\Big[(1-\kappa_t)\,\delta_{\blank}(\Z_t^i)
   + \kappa_t\,\delta_{\X_1^i}(\Z_t^i)\Big] \notag
\end{align}
with $\kappa_0=0,\ \kappa_1=1$. Each token in $\Z_t$ is blank with probability $1-\kappa_t$ or equals $\X_1^i$ with probability $\kappa_t$.

\begin{figure}
    \centering
\includegraphics[width=.45\textwidth]{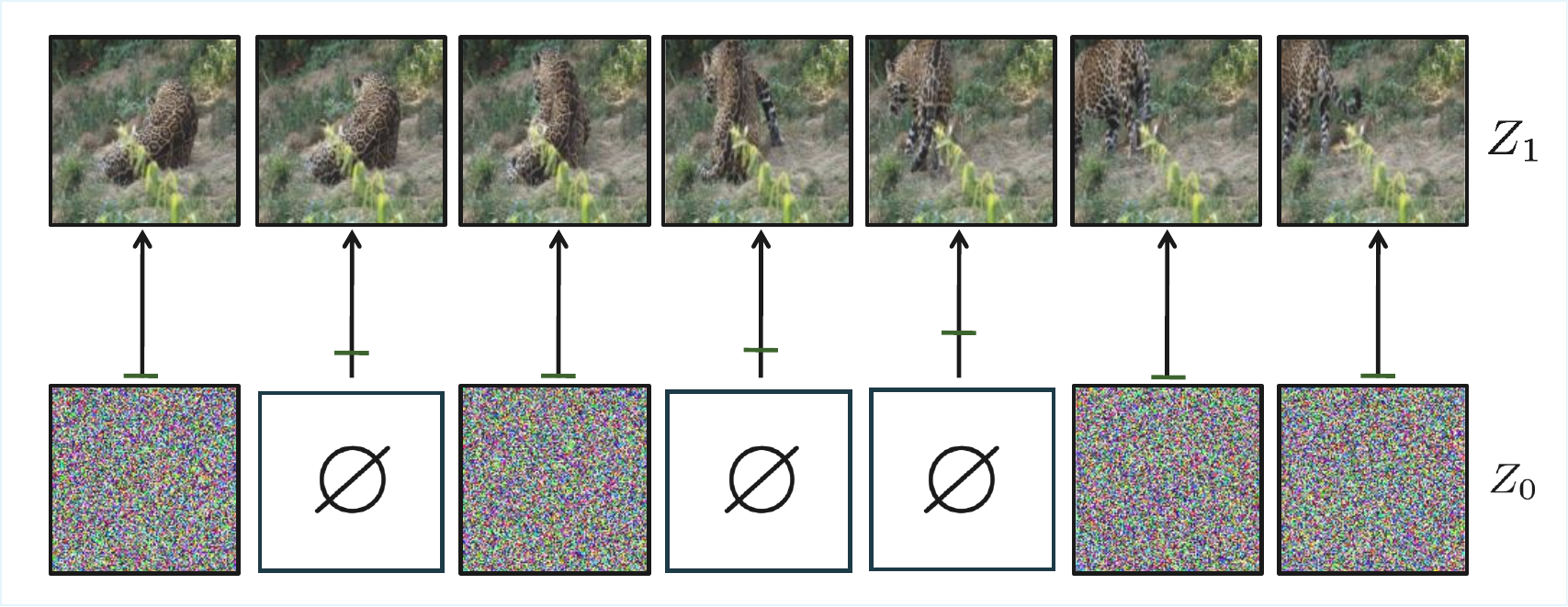}
    \caption{{\bf Illustration of the coupling between source and target distributions in augmented space $\mathcal{Z}$.} Starting frames are initialized as noise vectors $\epsilon \sim \mathcal{N}(0,I)$ while others are blank tokens in augmented space which are transformed into noise vectors at their corresponding insertion time.
    Horizontal lines on the arrows indicate the time when a frame is revealed in the schedule.
    }
    \label{fig:flowception_coupling}
\end{figure}

\mypar{Continuous-time Markov chain (CTMC)} We describe the
binary reveal process for frame insertions by a CTMC in
augmented space. 
As described previously, let $Z_t \in \mathcal{Z}$ denote the
augmented sequence with blanks, and $X_t = f_{\mathrm{strip}}(Z_t)$
its observable subsequence. The CTMC acts only on the
discrete reveal decisions in $Z_t$; continuous evolution of
frame contents is handled separately by the flow-matching
ODE which we develop later on.

Marginally over $Z_t$, the induced evolution on observable
sequences has the infinitesimal transition kernel
\begin{equation}
  \mathbb{P}(X_{t+h} \mid X_t)
  = \delta_{X_t}(X_{t+h})
    + h\,u_t(X_{t+h} \mid X_t) + o(h), \tag{15}
\end{equation}
where $u_t$ is the marginal insertion rate obtained from the
underlying CTMC on $(X_t,Z_t)$.

\mypar{Conditional CTMC rate}
As demonstrated in \citep{havasi2025edit}, a conditional CTMC that samples from \eqref{eq:flowception_cond_pt} can be written
\begin{align}\label{eq:flowception_cond_ut}
    u_t(x, z \mid \X_t, \Z_t, \X_1) & =\\
    \bigg( \sum_{i=1}^{n} \frac{\dot{\kappa}_t}{1-\kappa_t}&\big[\delta_{\X_1^i}(z^i) -\delta_{\Z_t^i}(z^i)\big]\bigg)
    \delta_{\strip(z)}(x), \notag
\end{align}
where $x=\ins(\X_t,i,a)$ for some  $i\in [n], a\in\mathcal{X}$.
This gives the infinitesimal probability shift from $(\X_t,\Z_t)$ to $(x,z)$, restricted to next states $x$ that differ by \emph{one insertion} at most.
Intuitively, any masked position $i$ transitions from $\blank$ to $\X_1^i$ with a rate $\dot\kappa_t/(1-\kappa_t)$, ensuring that a sample started at $Z_0{=}[\,\blank,\dots,\blank\,]$ reaches $\X_1$ as $t\to 1$.

\mypar{Training loss}
We train a model that transports sequences via insertions,
$u_t^\theta(x\mid \X_t)$, where $x=\ins(\X_t,i,\epsilon)$ for some $i,\ \epsilon$,
by marginalizing the auxiliary $Z_t$ and the data $\X_1$.

As shown by \cite{havasi2025edit}, the marginalized ground-truth rate $\bar u_t(\cdot\mid \X_t)=\sum_z \E_{p_t(z_t\mid \X_t)}\,u_t(\cdot,z\mid \X_t,z_t,\X_1)$ \emph{generates} $p_t(\X_t)$, any Bregman divergence $D_\phi(a,b)=\phi(a)-\phi(b)-\langle a-b,\nabla\phi(b)\rangle$ can be used to regress the marginal rate.
Following \cite{havasi2025edit,nguyen2025oneflowconcurrentmixedmodalinterleaved,holderrieth2024generator}, we use a Bregman divergence between measures over next states and marginalize over all $z$ such that $x=\strip(z)$:
\begin{multline}\label{eq:flowception_aux_bregman}
    \E_{\X_1\sim p_{\text{data}}} E_{(\X_t,\Z_t)\sim p_t(\X_t,\Z_t\mid \X_1)}\\
    \ D_\phi\!\left(\sum_{z} u_t(\cdot, z \mid \X_t,\Z_t,\X_1), u_t^\theta(\cdot \mid \X_t)\right).
\end{multline}
Choosing the entropy as a potential function $\phi(u)=\langle u,\log u\rangle$ gives the explicit loss (see \Cref{prop:bregman} for the derivation of this term)
\begin{align}
\mathcal{L}
&= \E_{t,\X_1,\X_t,\Z_t}\Bigg[
    \sum_{\mathclap{x\neq \X_t}} u_t^\theta(x \mid \X_t)
\label{eq:flowception_original_loss}\\[-2pt]
& -
    \sum_{i=1}^{n}
    \mathbf{1}(\Z_t^i=\blank)\,
    \frac{\dot{\kappa}_t}{1-\kappa_t}\,
    \log u_t^\theta\!\left(
        \ins\!\left(\X_t, j, \X_1^i\right) | \X_t
    \right) \Bigg], \notag
\end{align}
where $j$ is the slot in $\X_t$ corresponding to the first non-$\blank$ coordinate to the left of $Z_t^i$. 
This alignment ensures that inserting at position $i$ corresponds to changing $\Z_t^i:\blank \to \X_1^i$.

\mypar{Loss simplification}
Following  \cite{nguyen2025oneflowconcurrentmixedmodalinterleaved}, we adopt
an equivalent $t$-independent parameterization for one-insertion moves.
Rather than modeling separate rates for each individual missing frame, we
define a single slot-level insertion rate. For a one-insertion move
$x = \ins(\X_t,i,\epsilon)$ (inserting a fresh noise frame $\epsilon$ at
slot $i$), we write
\begin{equation}
    u_t^\theta\big(\ins(\X_t,i,\epsilon)\mid \X_t\big)
    =
    \frac{\dot{\kappa}_t}{1-\kappa_t}\,\lambda^i(\X_t),
\end{equation}
where $\lambda^i(\X_t)\ge 0$ is the \emph{total insertion rate} at slot
$i$. The actual frame content is always sampled from a fixed noise prior
(\eg $\epsilon \sim \mathcal{N}(0,I)$) and is not parameterized.

Let $\mathcal{A}_j$ denote the set of missing frames associated with
slot $j$, \ie those to be inserted ``to the right'' of $j$ at time $t$ under
the alignment. 
Substituting this parameterization into
\Cref{eq:flowception_original_loss} and collecting the
$\theta$-dependent terms yields
\ifcvpr
{
    \begin{align}
    \mathcal{L}(\theta)
    &=
     \E_{(\cdot)}\!\frac{\dot{\kappa}_t}{1-\kappa_t}\!
       \sum_{j=1}^{\ell(\X_t)}\left(\! \lambda^j(\X_t)\!
       -\!\sum_{a\in\mathcal{A}_j}
       \log\!\lambda^j(\X_t)
     \right)
    \nonumber \\
    &=
     \E_{(\cdot)}\!\frac{\dot{\kappa}_t}{1-\kappa_t}\!
     \sum_{j=1}^{\ell(\X_t)}
     \Big(
       \underbrace{\lambda^j(\X_t)-|\mathcal{A}_j|\log \lambda^j(\X_t)}_{\text{Poisson NLL}}
     \Big)
     + \text{const.}
    \end{align}
}
\else
{ 
    \begin{align}
    & \mathcal{L}(\theta)
    =
     \E_{(\cdot)}\!\frac{\dot{\kappa}_t}{1-\kappa_t}\!
       \sum_{j=1}^{\ell(\X_t)}\left(\! \lambda^j(\X_t)\!
       -\!\sum_{a\in\mathcal{A}_j}
       \log\!\lambda^j(\X_t)
     \right)
    \nonumber \\
    &=
     \E_{(\cdot)}\!\frac{\dot{\kappa}_t}{1-\kappa_t}\!
     \sum_{j=1}^{\ell(\X_t)}
     \Big(
       \underbrace{\lambda^j(\X_t)-|\mathcal{A}_j|\log \lambda^j(\X_t)}_{\text{Poisson NLL}}
     \Big)
     + \text{cte.}
    \end{align}
}
\fi
where the expectation is over
$\tau, \X_1\sim p_{\text{data}}$ and $(\X_t,\Z_t)\sim p_t(\X_t,\Z_t\mid \X_1)$, while $t = \text{clip}(\tau, 0, 1)$.
As pointed out by \cite{havasi2025edit}, keeping the factor
$\frac{\dot{\kappa}_t}{1-\kappa_t}$ preserves a direct ELBO
interpretation, while removing it can be more stable in practice; both
choices recover the familiar Poisson loss over insertion counts.
So in practice, the loss we use is
\begin{equation}
\label{eq:final_insertion_loss}
\mathcal{L}_{\mathrm{ins}}(\theta)
=
\mathbb{E}_{(\cdot)}
\Bigg[
  \sum_{j=1}^{\ell(\X_t)}
  \big(
    \lambda^j(\X_t) - |\mathcal{A}_j|\log \lambda^j(\X_t)
  \big)
\Bigg].
\end{equation}

\subsection{Velocity flow-matching objective}
Now we have the loss for the insertion term (where to insert and with with what probability), now we need to model how to update the currently active frames.
We briefly recall the derivation of the velocity loss
used for denoising active frames.

Following rectified flow matching, for $t_i \in [0,1]$, each frame $i$ follows a a linear probability path $t_i$
\begin{equation}
    \X^i_{t_i} = t_i \X^i_1 + (1-t_i) \X^i_0 \sim p^i_t.
\end{equation}
Consider a rectified flow coupling between source
$\X_0 \sim p_0$ and target $\X_1 \sim p_1$ governed by the
ODE
\begin{equation}
  \frac{d\X_t}{dt} = v(\X_t, t),
  \qquad t \in [0,1].
\end{equation}
Under this coupling, the population-optimal velocity field
satisfies
\begin{equation}
  v^\star(\X_t,t)
  =
  \mathbb{E}\left[ \X_1 - \X_0 \,\bigm|\,
    \X_t, t \right].
\end{equation}
The Conditional Flow Matching (CFM) objective then trains a neural
velocity field $v_\theta$ to regress onto this target:
\begin{equation} \label{eq:vel-loss-cfm}
  \mathcal{L}_{\mathrm{vel}}
  =
  \mathbb{E}_{\tau,X_0,X_1}
  \Bigl[
    \mathbf{1}_{[0,1)}(\tau)\,
    \bigl\|
      v_\theta(\X_t,t)
      - (\X_1 - \X_0)
    \bigr\|^2
  \Bigr],
\end{equation}
where $t = \text{clip}(\tau,0,1)$.
As shown by
\cite{lipman2023flow}, this CFM loss has the same
optimum as the original Flow Matching objective and is minimized
uniquely by $v_\theta = v^\star$ in function space.

In Flowception we apply this objective at the \emph{frame} level.
Each frame $i$ has a local time $t_i \in [0,1)$ induced by the
extended-time construction described above, and we write
$\X^i_{t_i}$ for its state along the linear path. The loss
in \Cref{eq:vel-loss-cfm} is evaluated only on \emph{active} frames,
\ie those with $\tau_i \in [0,1]$, and we mask out both frozen
frames ($\tau_i < 0$) and terminal frames ($\tau_i \geq 1$) during
training and sampling.

\subsection{Interleaved time schedule for frame insertions}
\label{app:interleaved_schedule}

We now derive the interleaved schedule used to concurrently insert new frames and denoise existing frames, ensuring that training and sampling observe the same joint law over local times.

\mypar{Design choice}
At the instant of insertion we can either (i) fully or partially denoise the frame, or (ii) insert pure noise and denoise afterwards. 
We adopt (ii) for concurrency and parallelism: a single forward pass handles both velocity prediction for present frames and insertion decisions, while newly inserted frames start at local time $0$ and are denoised in context thereafter.

Since the source $p_\text{src} = \mathcal{N}(0, I)$ and target $p_\text{data}$ distributions of the data are usually decoupled,  any sample from $p_\text{src}$ is valid for the inserted frame.

\mypar{Global and local times}
Let $t_g\in[0,1]$ denote the \emph{sequence} (global) time that advances monotonically during generation.
Each frame $i$ has a local time $t_i\in[0,1]$ (used by the rectified flow coupling).
We must respect the causal constraint that a frame cannot be more denoised than the sequence has progressed, i.e.,
$t_g \ge t_i$ for all frames $i$ currently present.

\mypar{Insertion-time law and inverse-CDF sampling}
Let $\kappa:[0,1]\to[0,1]$ be the monotone reveal scheduler with $\kappa_0=0,\,\kappa_1=1$, and define the hazard $\rho_\kappa(t)=\dot{\kappa}_t/(1-\kappa_t)$.
We model insertion times by the density $p(t_{\mathrm{ins}})=\dot{\kappa}_t$, 
 equivalently 
$t_{\mathrm{ins}}=\kappa^{-1}(u),\ \ u\sim\mathrm{Unif}(0,1)$.

For a frame inserted at sequence time $t_g$, we set its local time to zero: $t_i \leftarrow 0$.
Hence the instantaneous offset between the global and local times is distributed as
\begin{equation}
t_g-t_i = t_{\mathrm{ins}}, \qquad 0\le t_g,t_i,t_{\mathrm{ins}}\le 1,
\label{eq:offset_relation}
\end{equation}
so that local time always lags behind global time by a random, scheduler-consistent delay.

\mypar{Multi-frame generalization}
To make \Cref{eq:offset_relation} hold during training for all frames, we lift time to an extended interval and tie each frame to an independent offset.
Let
\begin{equation}
\tau_{g}\in[0,2], 
\qquad
t_g=\clip(\tau_g),
\qquad
t = \clip(\tau, 0, 1).
\label{eq:extended_global}
\end{equation}
For every potential frame index $i$, draw $u_i\sim\mathrm{Unif}(0,1)$ and define the \emph{extended local time}
\begin{equation}
\tau_i = \tau_g - \kappa^{-1}(u_i),
\qquad
t_i = \clip(\tau_i).
\label{eq:extended_local}
\end{equation}
This yields three phases per frame:
\begin{itemize}
    \item (frozen): $\tau_i<0$
    \item (flowing): $\tau_i\in[0,1]$
    \item (terminal): $\tau_i>1$
\end{itemize}

Let $A=\{i:\tau_i\in[0,1]\}$ denote the active set $\tau_i\in[0,1)$. 
During training we sample $(\tau_g,\{u_i\})$, form $(t_g,\{t_i\})$, delete frames with $\tau_i<0$, and apply the flow matching loss only to indices in $A$.
By construction, the marginal law of $(t_g,\{t_i\})$ exactly matches that encountered at sampling, thus samples remain in-distribution.

\section{Additional derivations and remarks}
\subsection{Deriving the insertion rate}
Let $\kappa:[0,1]\!\to\![0,1]$ be a nondecreasing insertion schedule with $\kappa(0)=0$, $\kappa(1)=1$,
differentiable almost everywhere. For a single frame, define its reveal time $T$ with cumulative density function (CDF)
$F_T(t)=\mathbb P[T\le t]=\kappa(t)\ \text{on }[0,1]$ and survival $S(t)=1-F_T(t)=1-\kappa(t)$.

\begin{lem}[Instantaneous reveal hazard]
\label{lem:hazard}
The instantaneous reveal hazard for insertion schedule $\kappa$ is given by
\begin{equation}
\begin{aligned} \label{eq:hazard_lim}
\rho_\kappa(t) &= \lim_{\Delta\to 0^+}
       \frac{\mathbb P \big(T \in [t, t + \Delta) \mid T > t \big)}{\Delta} \\
&= \frac{\dot\kappa(t)}{1-\kappa(t)}, 
\end{aligned}
\end{equation}
for $t\in(0,1)$ 
and $\rho_\kappa(t)=0$ outside $(0,1)$.
\end{lem}

\begin{proof}
By definition, for $\left[t, t+\Delta \right) \subset [0,1)$, we have
\begin{equation}
\begin{aligned}
\mathbb P(T\in[t, t + \Delta)\mid T > t) &= \frac{\mathbb{P}(T \in [t, t+ \Delta), T>t)}{\mathbb{P}(T>t)}\\
& = \frac{F_T(t + \Delta) - F_T(t)}{S(t)}\\
& = \frac{\kappa(t + \Delta) - \kappa(t)}{1 - \kappa(t)}.
\end{aligned}    
\end{equation}
Plugging this in \Cref{eq:hazard_lim} results in $\rho_\kappa(t)=\frac{\dot\kappa(t)}{1-\kappa(t)}$. 
For $t\notin(0,1)$, $F_T$ is constant, hence $\rho_\kappa(t)=0$.
\end{proof}

\subsection{From Bregman divergence to the insertion loss} \label{sec:ins_loss_derivation}
Let $\varphi(z)=z\log z - z$ on $\mathbb{R}_{\ge 0}$ and
\(
D_\varphi(u\Vert \lambda)=\sum_j \{u_j\log(u_j/\lambda_j)+\lambda_j-u_j\}.
\)

\begin{prop}[Bregman objective]
\label{prop:bregman}
For a fixed snapshot,
\[
\sum_j\big[\lambda_j - u_j\log \lambda_j\big]
\quad\text{and}\quad
D_\varphi(u\Vert \lambda)
\]
differ by a constant in $u$; thus they have the same minimizers. The training objective
\[
\mathcal{L}_{\mathrm{ins}}(\theta)=\mathbb{E}_{(\cdot)}\Big[\sum_j \lambda_{t,j}(X_t) - u_{t,j}(X_t)\log \lambda_{t,j}(X_t)\Big]
\]
is therefore a valid form to learn that converges towards the marginal expectation while only the conditional expectation is used.
\end{prop}


\begin{cor}[Pointwise optimum]
\label{cor:pointwise}
The per-snapshot objective in Prop.~\ref{prop:bregman} is strictly convex in each $\lambda_j>0$ and is minimized uniquely at $\lambda_j=u_{t,j}$.
\end{cor}

\begin{proof}
Similarly to Edit Flows~\citep{havasi2025edit}, we make use of the cross-entropy generator function to define the potential function for the Bregman divergence :
\[
\forall (z\ge 0), \quad \phi(z)=z\log z - z.
\]
Given ground-truth nonnegative targets $u=\{u_j\}_j$ and model predictions
$\lambda=\{\lambda_j\}_j$, the separable Bregman divergence is
\[
D_\phi(u\Vert \lambda)
=
\sum_j\Big(\phi(u_j)-\phi(\lambda_j)-\phi'(\lambda_j)\,(u_j-\lambda_j)\Big).
\]
Which simplifies for each coordinate $j$ to
\begin{equation}
\resizebox{\linewidth}{!}{$
\begin{aligned}
&\phi(u_j)-\phi(\lambda_j)-\phi'(\lambda_j)(u_j-\lambda_j)\\
=& \big(u_j\log u_j - u_j\big) - \big(\lambda_j\log \lambda_j - \lambda_j\big)
- (\log \lambda_j)\,(u_j-\lambda_j).
\end{aligned}
$}
\end{equation}

Collecting terms gives
\begin{equation}
\resizebox{\linewidth}{!}{$
\begin{aligned}
D_\phi(u\Vert \lambda)
& = \sum_j \Big(u_j\log \tfrac{u_j}{\lambda_j} + \lambda_j - u_j\Big)\\
& = \sum_j \Big(\lambda_j - u_j\log \lambda_j\Big) + \underbrace{\sum_j \big(u_j\log u_j - u_j\big)}_{\text{constant in }\lambda}.
\end{aligned}
$}
\end{equation}
Since $\sum_j (u_j\log u_j - u_j)$ does not depend on $\lambda$, minimizing
$\mathbb{E}[D_\phi(u\Vert \lambda)]$ over model parameters is equivalent to minimizing $\mathbb{E}\Big[\sum_j \lambda_j - u_j\log \lambda_j\Big]$.

For the reveal schedule $\kappa:[0,1]\!\to\![0,1]$ with hazard
$\rho_\kappa(t)=\dot\kappa(t)/(1-\kappa(t))$, the true marginal insertion rate at snapshot $(X_t,t)$ is
\[
u_{t,j}(X_t) = \rho_\kappa(t)\,K_j(X_t)\,\mathbf{1}_{[0,1]}(t),
\]
where $K_j(X_t)$ is the pending-count in slot $j$.
The general form of the insertion loss then becomes
\begin{equation}
\resizebox{\linewidth}{!}{$
\begin{aligned}
\mathcal{L}_{\mathrm{ins}}(\theta)
 & =\mathbb{E}\!\left[
      \sum_j \lambda_{t,j}(X_t)
      - u_{t,j}(X_t)\,\log \lambda_{t,j}(X_t)
    \right] \\
& = \mathbb{E}\!\left[
      \sum_j\!\lambda_{t,j}(X_t)\!-\! \rho_\kappa(t)\!K_j(X_t)\!\mathbf{1}_{[0,1]}(t)\!
      \log\!\lambda_{t,j}(X_t)\!\right]\!.
\end{aligned}
$}
\end{equation}

In particular, when using the linear scheduler $\kappa(t)=t$, then $\rho_\kappa(t)=\tfrac{1}{1-t}$ and

\begin{equation}
\resizebox{\linewidth}{!}{
$\mathcal L_{\mathrm{ins}}(\theta)\!=\!
\mathbb{E}\!\left[\!\sum_j\! \lambda_{t,j}(X_t)
  -
  \frac{K_j(X_t)\,\mathbf{1}_{[0,1]}(t)}{1-t}
  \log \lambda_{t,j}(X_t)
\right]$}
\end{equation}
\end{proof}

\subsection{Generalized insertions with Poisson thinning}
\label{sec:poisson-thinning}

\mypar{From Bernoulli to Poisson}
The Bernoulli-thinning sampler in  \Cref{alg:flowception_sampling} draws a single bit per slot and per step, which permits at most one insertion in slot $j$ during the step of size $\Delta$.
As $\Delta\!\to\!0$, the sum of independent Bernoulli micro-trials with success probability
$1-\exp(-\lambda\,\Delta)$ converges in law to a Poisson random variable with mean
$\int \lambda(s)\,ds$. This suggests a finite-step scheme that explicitly allows multiple insertions.

\mypar{Poisson process}
A time-inhomogeneous Poisson process on $[t,t+\Delta)$ with instantaneous rate $\rho(s)\ge 0$ has independent increments and satisfies
\[
N([t,t+\Delta)) \;\sim\; \mathrm{Poisson}\!\left(\int_t^{t+\Delta}\rho(s)\,ds\right).
\]
If $\rho(s)$ is approximately constant on the step, then
$N \sim \mathrm{Poisson}(\rho(t)\,\Delta)$.
Conditioned on $N=n$, the event times are i.i.d.\ uniform in the interval.

\mypar{Drawing from Poisson process}
To enable multiple insertions per slot, we replace the Bernoulli draw with a Poisson draw
\begin{align}
N_{t,j}\,\big|\,(\X_t,t)\;&\sim\;\mathrm{Poisson}\!\big(\Lambda_{t,j}\big),\\ \nonumber
\Lambda_{t,j}\;&\approx\;\Delta\,u_{t,j}(\X_t)\;=\;\Delta\,\rho_\kappa(t)\,K_j(\X_t).
\end{align}
This enables us to insert $N_{t,j}$ new elements into slot $j$ (each initialized with independent base noise) and doing this in parallel across all slots. 
This preserves the expected number of births per step, $\mathbb{E}[N_{t,j}]\approx \Delta\,u_{t,j}(\X_t)$, and removes the at most one insertion per slot constraint.

\section{Discussions}
\subsection{Baseline implementation details}

\mypar{Full-sequence}
For the full-sequence model training, we follow standard setups~\citep{wan2025, HaCohen2024LTXVideo} and sample timesteps during training according to a lognorm schedule.
Similarly to Flowception, we expand the channel dimension of the input by a factor two, and use the second half to encode the (clean) context frames, allowing to support the image-to-video framework.
For training on variable length videos, we experiment with two strategies. 
(1) In each batch we have  videos of different length, and we mask out the loss on padded frames. (2) in each batch we collect videos of the same length only. 
We find (2) to perform more favorably in preliminary experiments and consequently use it for our remaining experiments.

\mypar{Autoregressive}
We also use a lognorm schedule for the timestep sampler, while the number of context frames is sampled randomly with uniform probability across the length of the video.

\subsection{Finetuning Pre-trained Models}
In order to finetune open-source models with our method, there are some changes that need to be done to the architecture to adapt to our framework, we describe these changes below.
\begin{itemize}
    \item Framewise time embeddings: In \ours, each frame in the video can have different noise level associated to a timestep $t_i$, we therefore adapt the AdaLN layers to modulate each frame independently according to its noise level.
    \item Variable length: we extend the model to support variable length sequences per-batch by feeding a frame activity mask which is responsible for the masking padding frames from the attention.
    \item Insertion rate tokens: We append one learnable frame token per frame to the sequence (after patch embedding and before the transformer stack). This token participates in all self and cross-attention layers, takne the rope coordinates corresponding to the gap between layers at the center spatial position $(c_x, c_y, c_t = W/2, H/2, t_i)$, allowing it to aggregate frame-level information to predict the insertion rates. 
    After the transformer layers, a lightweight head followed by an exponential activation is used to predict the insertion rates.
\end{itemize}

\subsection{\ours as implicit temporal compression}

When a frame is inserted as pure noise $(x_{\text{new}}=\varepsilon,\ t_{\text{new}}=0)$, its clean identity among the ($K$) pending in-slot frames is unresolved at birth. 
Under the masked Flow Matching objective, the population-optimal first velocity is the conditional mean {\it over the posterior of the missing frames} (integrating both which clean frame it will become and that frame’s content):
\begin{align}
    v^{\star}_{\text{new}} & = \mathbb{E}_{\text{miss}} \big[X_{1,\text{new}}\mid X_t,t,M\big] - \varepsilon \\
    & =\mathbb{E}_{\text{miss}}\big[Z\mid X_t,t,M\big]-\varepsilon,
\end{align}
where $Z$ is a random clean frame drawn from the posterior over the $K$ not-yet-revealed frames in that slot induced by the snapshot law and the insertion-rate hazard $u
_{t,j}(X_t)=\rho_\kappa(t),K_j(X_t)$, with $\rho_\kappa(t)=\dot\kappa(t)/(1-\kappa(t))$, and $\rho_\kappa(t)=1/(1-t)$ for linear $\kappa$. 
Thus the first update points from noise toward a {\it group-wise conditional expectation over the missing frames}, not toward a single target. 

In full-sequence flow matching, by contrast, all frames are present (as noise) at every step, so there is no identity ambiguity: the optimal direction for index $j$ is the per-index conditional mean $ \mathbb{E}[X_{1,j}\mid X_t,t]-X_{0,j}$, \ie, no marginalisation over missing content. 
The Flowception update therefore acts like an implicit temporal aggregator early on: 
active tokens move toward expectations that  average over as-yet unseen in-between motion, while the unrevealed frames are integrated out.

\section{Efficiency Comparison}

We now study the impact of using \ours on sampling efficiency compared to full-sequence diffusion and autoregressive paradigms. Let $L=H\!\times\!W$ denote the number of tokens per frame, and $n$ the number of frames. The total sequence length is therefore $nL$.
For this analysis we disregard the text tokens and the per-frame extra rate token for \ours since we are only interested in orders of magnitude.

\mypar{Full-sequence}
Full-sequence diffusion and flows evolve all frames simultaneously, sharing a global timesteps between the $n$ frames.
At each sampling step all $n L$ tokens are active, and self-attention dominates with quadratic cost. The total complexity over $T_{\text{full}}$ steps is therefore
\[
\mathcal{C}_{\text{full}} \approx T_{\text{full}} \,(n L)^2.
\]

\mypar{Autoregressive (no caching)}
In autoregressive diffusion/flow-matching, frames are generated sequentially, one at a time, each conditioned on all previously generated frames.
A generation step involves appending a noise frame to the end of the sequence before evolving it using flow matching, which requires $T_\text{AR}$ inner steps in order to evolve the noise sample into a valid frame by predicting the velocity field for that frame.
At step $j$, the active sequence length is $jL$, yielding a cumulative cost
\[
\mathcal{C}_{\text{AR}} \approx T_{\text{AR}} \sum_{j=1}^{n} (jL)^2
\leq \tfrac{1}{3}\,T_{\text{AR}}\,L^2 n^3.
\]
This cubic dependence on $n$ makes autoregressive diffusion substantially more expensive than full-sequence diffusion.

\mypar{Autoregressive diffusion (with caching)}
With key-value (KV) caching, past tokens do not need to be recomputed. At step $j$, attention is computed only between the $L$ new tokens and the cached $jL$ past tokens, for cost $\mathcal{O}(j L^2)$. Summing across $n$ frames yields
\[
\mathcal{C}_{\text{AR+cache}} \approx T_{\text{AR}} \sum_{j=1}^n j L^2
\approx \tfrac{1}{2}\,T_{\text{AR}}\,L^2 n^2.
\]

\mypar{Flowception}
When starting from the empty sequence and under a linear insertion scheduler, the active fraction is $\kappa(\tau){=}\tau$ and the (expected) active sequence length art any point is $R_\tau{=}\tau n L$. 
Averaging the quadratic self-attention cost over the trajectory yields
\[
\mathcal{C}_{\text{flow}} \approx T_{\text{flow}}\, (n L)^2 \,\mathbb{E}_\tau[\tau^2]
= \tfrac{1}{3}\,T_{\text{flow}}\, (n L)^2.
\]
Allowing Flowception to take $\alpha$ times more steps than the baseline ($T_{\text{flow}}=\alpha T_{\text{full}}$),  to account for the delayed denoising of frames interted later, gives
\[
\mathcal{C}_{\text{flow}} \approx \tfrac{\alpha}{3}\,T_{\text{full}}\, (n L)^2.
\]
In our experiments we set $\alpha=2$ to roughly allow the same number of denoising steps per frame as the full-sequence model, even for frames inserted close to $t_g=1$.

\mypar{Comparison}
The asymptotic speedups between the different methods are
\begin{align*}
\text{speedup}_{\text{FC vs. Full}} & = \frac{\mathcal{C}_{\text{full}}}{\mathcal{C}_{\text{flow}}}
\approx \frac{3}{\alpha}\\
\text{speedup}_{\text{FC vs. AR}} & \approx \frac{n}{\alpha}\,\frac{T_{\text{AR}}}{T_{\text{full}}}\\
\text{speedup}_{\text{FC vs. AR+cache}} & \approx \frac{3}{2\alpha}\,\frac{T_{\text{AR}}}{T_{\text{full}}}
\end{align*}
\ours achieves computational efficiency by interleaving discrete frame insertions with continuous denoising, resulting in an implicit temporal compression that reduces the average sequence length per sampling step. Unlike AR without caching models which exhibit cubic complexity $\mathcal{O}(n^3)$ in the number of frames $n$, \ours scales quadratically, $\mathcal{O}(n^2)$, providing a significant asymptotic speedup. Compared to full-sequence diffusion, the method achieves a theoretical speedup of approximately $3/\alpha$ (yielding $1.5\times$ for $\alpha=2$) by leveraging a linear insertion schedule that avoids redundant computation on noise-only tokens.
We summarize the complexities of the different frameworks in \Cref{tab:complexity}.

\ifcvpr
{
    \begin{table}
    \caption{{\bf Comparison of expected FLOPs during sampling.} }
        \scriptsize
        \centering
        \begin{tabular}{lcccc}
            \toprule
              & Full-seq & AR & AR+cache & \ours \\ 
             \midrule
             Complexity & $T_\text{full} (n L)^2$ & $\frac{n}{3} T_\text{AR} (nL)^2$ & $\frac{1}{2} T_\text{AR} (nL)^2 $ & $\frac{\alpha}{3}   T_\text{full} (n L)^2$ \\
             \bottomrule
        \end{tabular}
        \label{tab:complexity}
    \end{table}
}
\else
{
    \begin{table}
    \caption{{\bf Comparison of expected FLOPs during sampling.} }
        \resizebox{\linewidth}{!}{
        \centering
        \begin{tabular}{lcccc}
            \toprule
              & Full-seq & AR & AR+cache & \ours \\ 
             \midrule
             Complexity & $T_\text{full} (n L)^2$ & $\frac{n}{3} T_\text{AR} (nL)^2$ & $\frac{1}{2} T_\text{AR} (nL)^2 $ & $\frac{\alpha}{3}   T_\text{full} (n L)^2$ \\
             \bottomrule
        \end{tabular}}
        \label{tab:complexity}
    \end{table}
}
\fi
\section{Algorithms \& implementation}
\label{app:algorithms}

We provide algorithms for \ours training and sampling procedures.

In \Cref{alg:flowception_sampling}, we provide a sketch of the sampling algorithm, assuming a number of starting frames $n_\text{start}$ and a step size $h$ that is shared between insertions and flow matching.
For simplicity, we do not include context frames in the sketch of the algorithms.
We start with $t_i=0$ for the  starting frames.
Each sampling step iterates two operations, flow matching on the current set of frames, followed by insertions to the right of each frame $i \in \{1, \dots, \ell(X)\}$, which happens with probability $h_i \lambda_i  \tfrac{\dot{\kappa}(t_g)}{1 - \kappa(t_g)}$ where $\lambda_i$ is the rate associated with the frame $i$.
When a new frame is inserted, we insert a new frame as a pure noise $\ins(X,i,\epsilon), \epsilon \sim \mathcal{N}(0, I)$.

\begin{algorithm}[t]
  \caption{Flowception generation procedure}
  \begin{algorithmic}[1]  
    \Function{FlowceptionGeneration}{step size $h$}
      \State $X \sim \prod_{i=1}^{n_\text{start}} \mathcal{N}(X^i ; 0, I)$
      \State $t \gets [0, \dots, 0] = [0]^{n_\text{start}}$
      \Comment{{\color{gray!50} per-frame times}}
      \State $t_g \gets 0$
      \Comment{{\color{gray!50} global time}}
      \While{$\min\{t_i\} < 1$} \Comment{{\color{gray!50} iterate until all frames are clean}}
        \State $X, t, t_g \gets \textsc{FlowceptionStep}(X, t, t_g, h)$
      \EndWhile
      \Return $X$
    \EndFunction
  \end{algorithmic}
  \begin{algorithmic}[1]  
    \Function{FlowceptionStep}{$X, t, t_g, h$}
      \State $v, \lambda \gets \texttt{FlowceptionModel}(X, t; \theta)$
      \State $h_i = \min\{h, 1 - t_i\}$
      \Comment{{\color{gray!50} clean frames are frozen}}
      \State $X \gets X + h v$ 
      \Comment{{\color{gray!50} apply flow step to denoise}}
      \State $t \gets \text{clip}(t + h, 0, 1)$
      \State $t_g \gets \text{max}(t)$
      \Comment{{\color{gray!50} update time trackers}}
      \State {$\triangleright$ {\color{gray!50}all insertions are implemented in parallel}}
        \ForAll{$i \in \{1,\dots,\ell(X)\}$} 
          \State \textbf{with probability} $h_i\lambda_i  \tfrac{\dot{\kappa}(t_g)}{1 - \kappa(t_g)}$: 
          \State{\qquad $\triangleright$ {\color{gray!50} inserted frames are set to pure noise}}
          \State \qquad $X = \ins(X, i, \varepsilon)$ where $\varepsilon \sim \mathcal{N}(0, I)$
          \State{\qquad $\triangleright$ {\color{gray!50} inserted time values are set to zero}}
          \State \qquad $t = \ins(t, i, 0)$
        \EndFor
        \State \Return $X, t$
    \EndFunction
  \end{algorithmic}    \label{alg:flowception_sampling}
\end{algorithm}

We detail the training algorithm in \Cref{alg:flowception_train}.
First, we sample a set of timesteps according to the \ours schedule, $\tau_i,\ i \in \{0, \dots,  n\}$, we map these timesteps to deletion operations according to the insertion schedule $X \leftarrow \strip(X_\text{target},M), t \leftarrow \strip(t, M)$.
Next, the remaining frames are noised according the rectified flow matching schedule $X = tX + (1-t)X_0$.
The model is then fed these noised frames and their associated timesteps, predicting their associated velocities and insertion rates.
The training losses are detailed in the main manuscript.

\begin{algorithm}[t]
  \caption{Flowception training procedure}
  \label{alg:flowception_train}
  \begin{algorithmic}[1]  
    \Require scheduler $\kappa$
    \Function{FlowceptionTrainingStep}{$\kappa$}
      \State $X_\text{target} \sim \ptarget$
      \State $\tau_g \sim p(\tau_g)$
      \Comment{{\color{gray!50} can choose \eg logit normal}}
      \State $u_i \sim \text{Unif}(0,1)$
      \State $\tau_i \gets \tau_g - \kappa^{-1}(u_i)$
      \Comment{{\color{gray!50} per-frame extended times}}
      \State $M_i \gets \1_{[\tau_i \geq 0]}$
      \Comment{{\color{gray!50} $i$-th frame is deleted if $\tau_i < 0$}}
      \State $\triangleright$ {\color{gray!50}sample noisy frames}
      \State $t_i \gets \clip(\tau_i)$
      \State $X_0 \sim \mathcal{N}(0, I)$
      \State $X = tX_\text{target} + (1-t)X_0$
      \State $\triangleright$ {\color{gray!50}  remove deleted frames}
      \State $X \gets \strip(X, M)$
      \State $t \gets \strip(t, M)$
      \State $v, \lambda \gets \texttt{FlowceptionModel}(X, t ; \theta)$
      \State $\mathcal{L} \gets \dots$
      \Comment{{\color{gray!50}compute insertion and velocity losses}}
      \State $\theta \gets \texttt{optimizer\_step}(\theta, \nabla \mathcal{L})$
    \EndFunction
  \end{algorithmic}
\end{algorithm}

\mypar{Time sampling}
using the time sampling in \Cref{alg:flowception_train}, it can happen that for the sampled $\tau_g$ all frames in a video are already denoised (in particular when all frames are inserted early), rendering the video useless for training. 
To ensure that there is at least one evolving frame per video, we can instead first sample the terminal time for the last frame in the video (when the last flow step happens), before deriving $\tau_g$ and sampling the individual offsets.
To do this we proceed in the following manner:
\begin{enumerate}
    \item Sample the insertion times $t_\text{ins}^i \sim \text{Unif}(0,1)$
    \item Compute maximum $\tau_\text{max}$ that we need: $\max\{ t_\text{ins}^i \} + 1$
    \item Sample $\tau_g \sim \text{Unif}(0, \tau_\text{max})$
    \item Compute $\tau_i = \tau_g -t_\text{ins}^i$
\end{enumerate}

\noindent
Additionally, following other works \citep{nguyen2025oneflowconcurrentmixedmodalinterleaved, wan2025, HaCohen2024LTXVideo}, we used a lognorm global schedule during training to be beneficial $\tau_g \sim \text{lognorm}(0,1) \cdot \tau_\text{max}$.

\mypar{Loss reduction}
Another important point is about the loss reduction for both the velocity and Poisson likelihood, since each sample in the training batch can have a varying number of frames which are still evolving, \ie with $0\leq\tau_i<1$, a question arises around how to reduce the velocity loss across these frames.
We experimented with both per-sample mean reduction (each video contributes equally independently from the number of active frames) and a mean reduction across all active frames in the batch (all frames in the batch contribute equally, so  shorter videos contribute less). 
We found the latter to be more stable, especially for longer sequences while the former tends to over-optimize for short sequences (early stages of sampling or shorter videos) which in turns biases the model to under-insert.

\mypar{Local sampling schedules}
In video generation, early diffusion timesteps are particularly important, as they establish global structure, motion, and temporal alignment before later steps primarily refine appearance. 
This effect is amplified in our setting, where frames are inserted asynchronously into an evolving sequence: immediately after insertion, a frame is highly uncertain and must rapidly become consistent with its temporal neighbors. 
To address this, the denoising time steps can be biased towards the start of the process.  For example, \cite{polyak2025moviegencastmedia} uses a linear-quadratic scheduler where the first portion of sampling follows a linear schedule with a small step size $\approx 1/1000$, while the remaining time steps follow a quadratic schedule to arrive at $t=1$.

To prioritize this post-insertion regime in \ours, we introduce a \emph{framewise time reparameterization}: each frame \(i\) maintains its own solver coordinate \(u_i \in [0,1]\) and physical diffusion time \(t_i \in [0,1]\), linked by a strictly increasing function \(t_i = f(u_i)\). During sampling, we advance the solver coordinates by a fixed step \(\Delta u\) for all active frames, compute the corresponding physical increments \(\Delta t_i = f(u_i + \Delta u) - f(u_i)\), and update \(x_i \leftarrow x_i + \Delta t_i\, v_\theta(x_i, t_i)\) in the same scalar time variable \(t \in [0,1]\) used during training. This preserves consistency with the training setup while allowing the effective step size in diffusion time to depend on a frame’s ``age'' since insertion. In practice, we instantiate this family with a power schedule \(t_i = u_i^\gamma\) (with \(\gamma > 1\)), which yields small \(\Delta t_i\) for newly inserted frames (small \(u_i\)) and larger \(\Delta t_i\) for older, well-established frames. 
As a result, the sampler allocates more computation to the crucial early timesteps after each insertion and fewer steps to later, easier denoising phases, leading to improved temporal coherence and fewer artifacts compared to a uniform schedule (\(\gamma = 1\)) under the same compute budget.

\section{Additional experiments}

In this section we provide additional experimental results to complement those in the main paper.

\mypar{Length modeling}
Here we  assess the ability of \ours to model the  length distribution of the sequences in the training set.
We create a toy dataset where the number of frames is either 15, 20, 25 or 30.
Each sample is a $3\time3$ pixel video where the middle pixel makes a discrete jump between two pixel values, while the boundary pixels make up a gradient between two colors that moves along the boundary with constant speed in order to achieve an integer number of rotations.
After training \ours on this dataset, we compare the histogram of ground truth and generated video lengths and plot them in \Cref{fig:toy_length}.
As expected the generated video lengths follow a similar distribution to the data distribution, with peaks around 15, 20, 25 and 30, while rarely generating videos with lengths outside these four modes. 

\begin{figure}
    \centering
    \includegraphics[width=.45\textwidth]{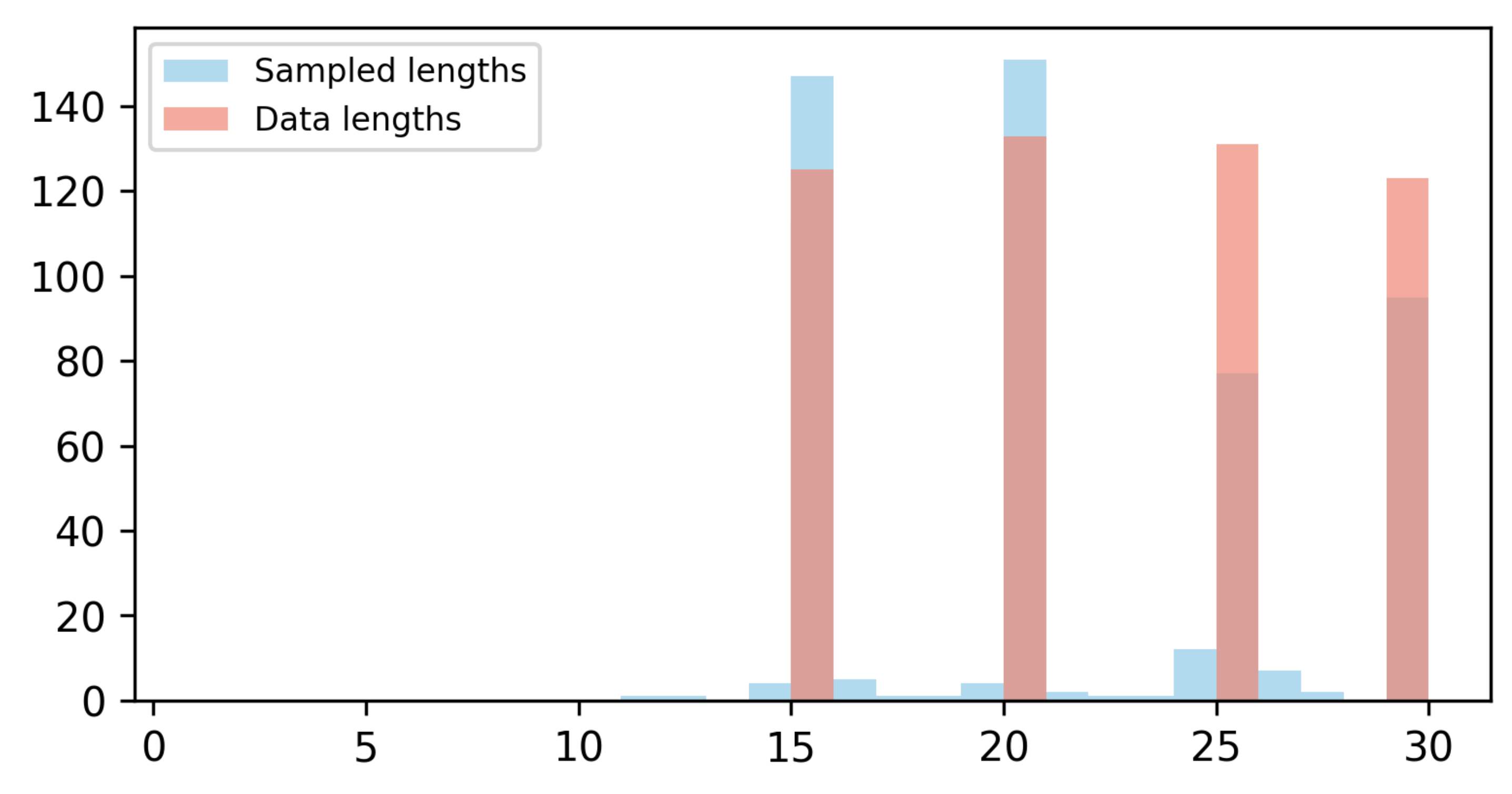}
    \caption{{\bf Video length matching.} Our framework is able to accurately reproduce the length of videos from the toy dataset.}
    \label{fig:toy_length}
\end{figure}

\mypar{Efficiency comparison}
To evaluate practical sampling efficiency, we perform a sweep over the number of sampling steps for both the autoregressive model and \ours to compute how FVD performance changes as a function of the number of sampling steps and FLOPs.
We report the  results in \Cref{fig:nfe_comparison}. 
First, we find that  the autoregressive model has a somewhat  lower number of FLOPs for a given number of sampling steps per frame, this is because \ours  denoise frames asynchronously so the total number of sampling steps is larger than the number of per-frame sampling steps. 
Second, we find that \ours  obtains  significantly better FVD  for a given number of FLOPs, with FVD plateauing at around  32 denoising  steps per frame, where the autoregressive model continues to improve at least up to 64 steps, but without closing the gap with \ours.
\ifcvpr
{
    \begin{figure} 
        \centering
        \includegraphics[width=.45\textwidth]{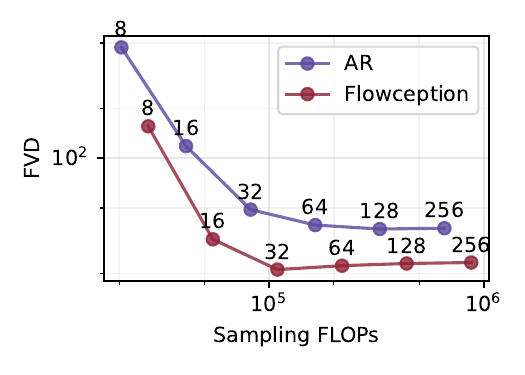}
        \caption{{\bf Efficiency comparison.} We compare the sampling efficiency of autoregressive model (with caching) with \ours, we plot the FVD on RealEstate10K as a function of the  FLOPs used for  sampling when varying the number of sampling steps.}
        \label{fig:nfe_comparison}
    \end{figure}
}
\else
{
    \begin{figure}
        \centering
        \includegraphics[width=.4\textwidth]{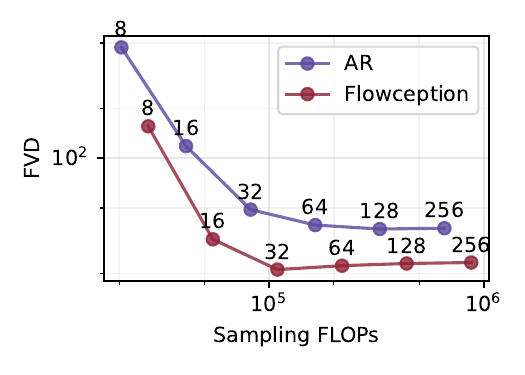}
        \caption{{\bf Efficiency comparison.} We compare the sampling efficiency of autoregressive model (with caching) with \ours, we plot the FVD on RealEstate10K as a function of the  FLOPs used for  sampling when varying the number of sampling steps.}
        \label{fig:nfe_comparison}
    \end{figure}

}
\fi

\mypar{Long video generation}
While our main experiments focus on clips within the model context length, \ours can be extended to longer videos with minimal changes. We consider two complementary approaches. First, we fine-tune the model on longer clips when available, which improves long-horizon stability while preserving short-clip performance. Second, to generate videos beyond the training/context window, we adopt a block-wise sampling strategy: we partition the target video into temporal blocks and sequentially sample each block conditioned on the previously generated frames. Concretely, we generate the first block from the text prompt, then iteratively generate the next block while conditioning on a short prefix of past frames (temporal overlap) to maintain appearance and motion continuity. This enables minute-scale rollouts without modifying the backbone architecture. 
We detail this sampling mechanism in \Cref{alg:flowception_chunked}.
In \Cref{app:qual} we provide examples of such long rollouts. 

\begin{algorithm}[t]
  \caption{Chunked Flowception generation}
  \label{alg:flowception_chunked}
  \begin{algorithmic}[1]
    \Function{ChunkedFlowceptionGeneration}{step size $h$, chunks $N$, window $L$, overlap $O$}
      \State Initialize an extended sequence $X$ with $n_\text{start}$ noisy frames (rest is padding)
      \State $t \gets [0,\dots,0]$ for valid frames \Comment{{\color{gray!50} per-frame times}}
      \State $t_g \gets 0$ \Comment{{\color{gray!50} global time}}
      \State Initialize chunk boundaries $\{[s_c,e_c)\}_{c=1}^N$ to cover valid frames (length $\le L$, overlap $O$)
      \While{$\min\{t_i\} < 1$} \Comment{{\color{gray!50} iterate until all frames are clean}}
        \State Extract chunk views $\{(X^{(c)},t^{(c)})\}_{c=1}^N$ using current boundaries $\{[s_c,e_c)\}$
        \ForAll{chunks $c \in \{1,\dots,N\}$ \textbf{in parallel}}
          \State $X^{(c)}, t^{(c)}, t_g \gets \textsc{FlowceptionStep}(X^{(c)}, t^{(c)}, t_g, h)$
        \EndFor
        \State Write chunk updates back into $X,t$ (blend overlap regions)
        \If{new frames were inserted in this step}
          \State Update chunk boundaries $\{[s_c,e_c)\}_{c=1}^N$ to again cover all valid frames (length $\le L$, overlap $O$)
        \EndIf
      \EndWhile
      \Return $X$
    \EndFunction
  \end{algorithmic}
\end{algorithm}

\section{Broader societal impact}
We recognize that our work could lead to potential negative societal impacts, as our method can help generate photorealistic videos, especially if combined with conditioning on real photos or videos.
Nonetheless, our work also paves the way for efficient and flexible video generation that can be beneficial in domains such as the entertainment or film industry, as well as world modeling frameworks.
As an example, animators could create coherent animations by providing a set of frames (either drawings or AI generated), which can significantly speed-up animation work flows.
Our method can be used to hierarchically generate very long videos of high quality, at a lower computational cost by adopting local attention variants, thereby reducing the energy footprint of generative video models.

\begin{figure*}[t]
    \centering
    \includegraphics[width=\linewidth]{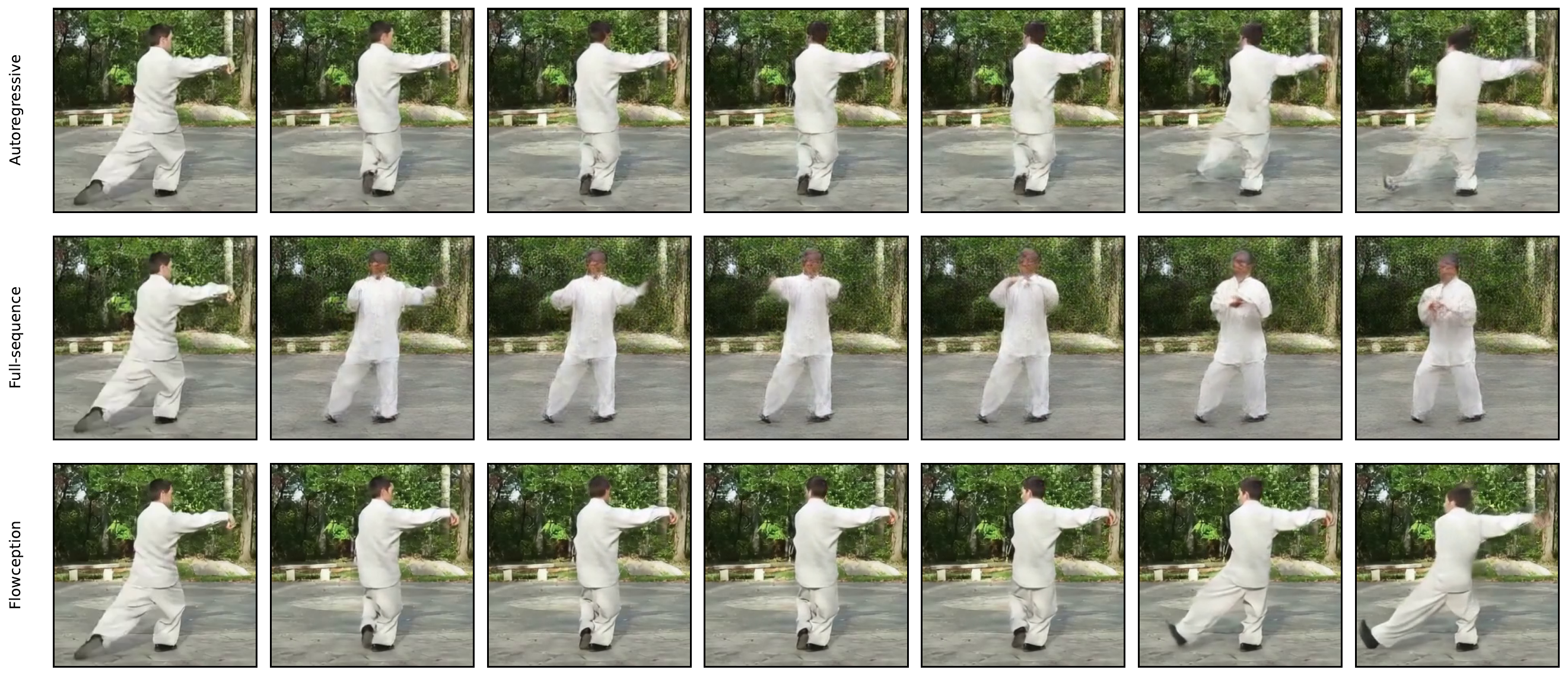}
    \caption{{\bf Comparing different methods using models trained on the Tai-Chi-HD dataset.} Using the same input frame (left) and random seed, we compare generations with the autoregressive, full-sequence and \ours models.}
    \label{fig:taichi_comp}
\end{figure*}

\begin{figure*}
    \centering
    \includegraphics[width=\linewidth]{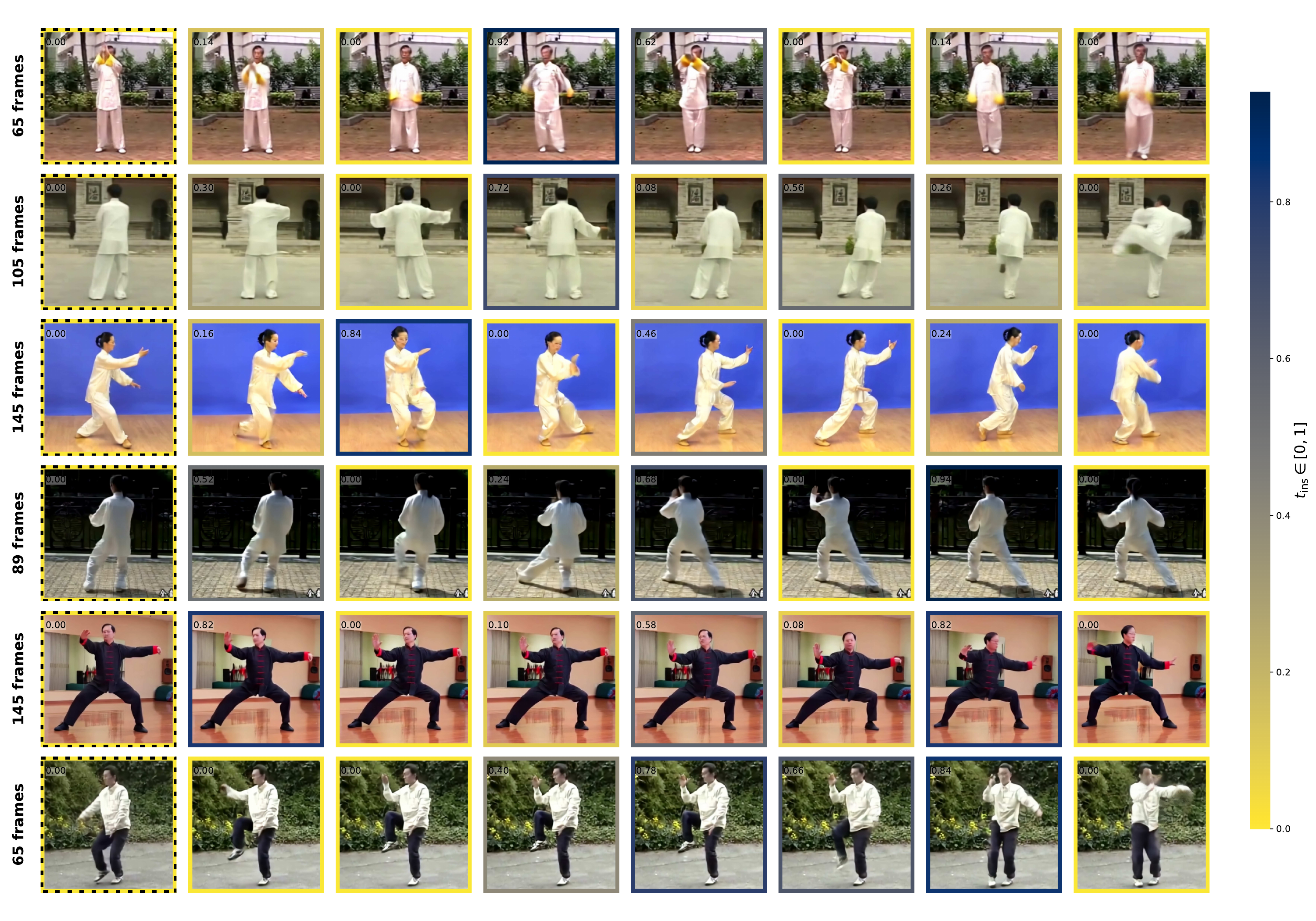}
    \caption{{\bf Additional qualitative examples on Taichi.} Each row corresponds to a different video obtained with our method for image-to-video generation}
    \label{fig:taichi_i2v_supp}
\end{figure*}

\section{Additional qualitative examples}
\label{app:qual}

The following are the captions used in \Cref{fig:teaser}:
\begin{enumerate}
    \item A yellow taxi moving through new york streets on a rainy night.
    \item A cute puppy wearing a yellow hat happily strolling in a field of roses.
\end{enumerate}

In \Cref{fig:t2v_appendix}, we provide additional qualitatives for the text-to-video generations obtained by finetuning the LTX model.
Below are the captions used:

\begin{enumerate}
    \item A compact humanoid robot with a smooth white helmet-like head, dark face panel, and a black armored body with glowing yellow/orange light accents walks through a busy New York City street at night. The scene feels like Times Square with bright billboards and crowds in the background.
    \item Pretty farmer. agriculture business concept. farmer girl examines the rose crops at sunset. farmer walk agriculture lifestyle rose concept. farmer works in a field with roses at sunset.
    \item Third person view of a man riding a boat in the sea towards a deserted island. A dolphin jumps in front of the boat.
    \item A wild horse walking by a lake. Beautiful scene.
    \item A cute bear in a knitted blue sweater and round glasses lounges on a sofa, reading a newspaper. A packed bookshelf and a crackling fireplace glow warmly in the background.
    \item A lone reindeer stands on a windswept snowy mountain ridge at sunrise. The camera slowly pushes in, revealing vast icy peaks and a glowing sky while the reindeer takes a few careful steps through fresh powder.
\end{enumerate}

In \Cref{fig:t2v_chunk}, we provide examples of long generation rollouts using our chunked sampling algorithm.
We used the following prompts:
\begin{enumerate}
    \item A man wearing a black leather jacket walks through the crowded streets of Venice at noon, then arrives at a small park with a flowing water fountain.
    \item A continuous cinematic shot of a wolf walking across an open field toward a small lake. The camera follows steadily from the side as the wolf moves through the grass and arrives at the water.
\end{enumerate}

In \Cref{fig:taichi_comp} we compare generations of \ours and the autoregressive  and Full-Sequence baselines  trained on the Tai-Chi-HD dataset for 300k iterations, the generations are of 145 frames with an FPS of 16.
For the autoregressive model, we observe that later frames suffer from drift  due to error accumulation, hindering their quality (see,\eg, the legs).
For the full-sequence model, the model struggles to accurately generate the high-frequency details of the video accurately (see, \eg, the face and foliage in the background).
In contrast, \ours results in a sharp video without error accumulation as the video progresses.
In \Cref{fig:taichi_i2v_supp} and \Cref{fig:qual_i2v_re10k} we provide further examples of image-to-video results obtained with \ours on the Tai-Chi-HD and RealEstate10k datasets respectively. 

In \Cref{fig:kinetics_2} and \Cref{fig:re10k_2}, we provide additional video interpolation results obtained with \ours on the  Kinetics 600 and RealEstate10K datasets, respectively; extending the results in \Cref{fig:kinetics_interp} of the main paper.
We always provide the first and last frames as context, plus at most  two additional intermediate frames.

\begin{figure*}
    \centering
    \includegraphics[width=\linewidth]{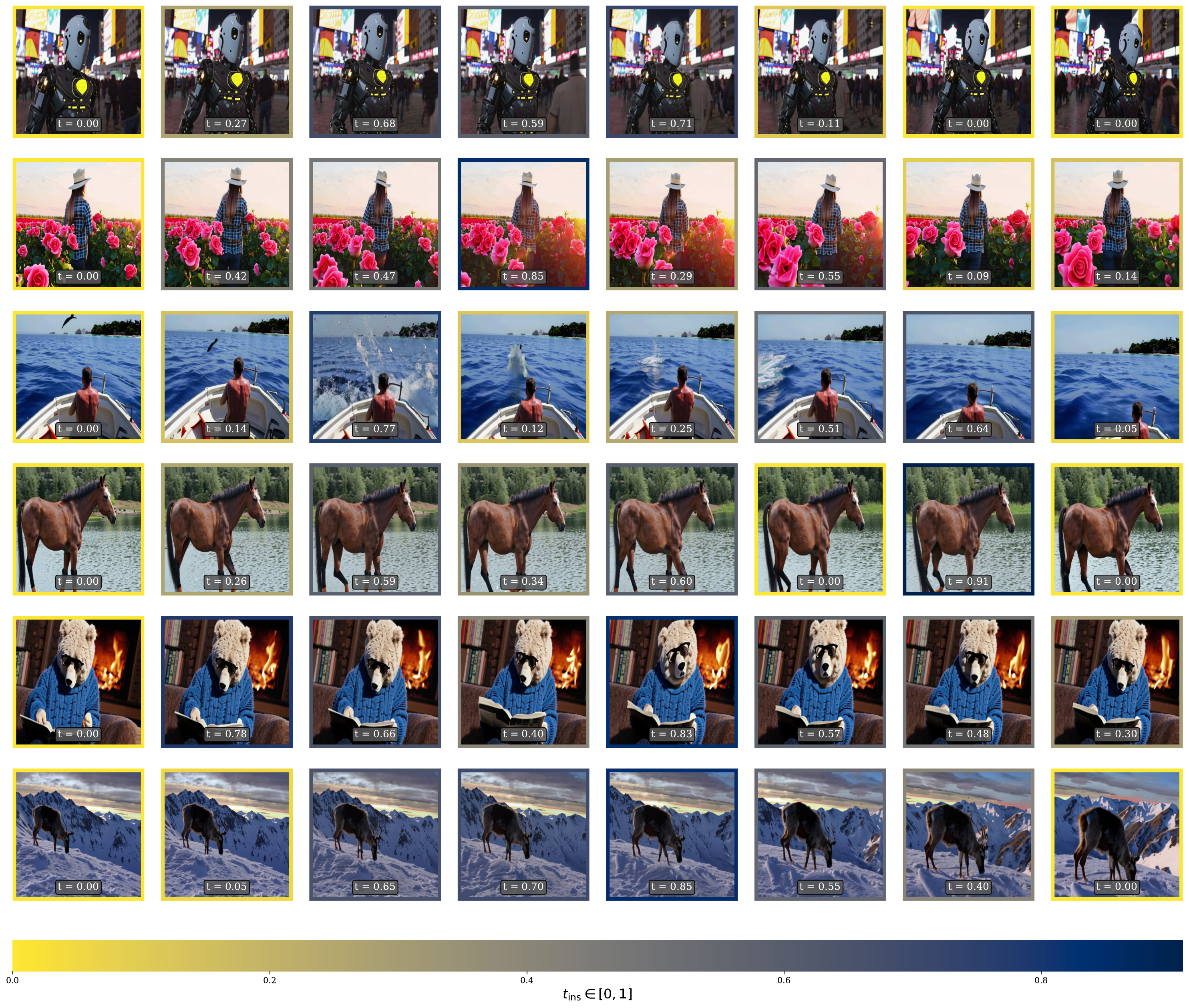}
    \caption{{\bf Text-to-video generations with \ours.} Additional qualitatives for the LTX-2B model finetuned on our internal data split of text-video pairs.}
    \label{fig:t2v_appendix}
\end{figure*}

\begin{figure*}
    \centering
    \includegraphics[width=\linewidth]{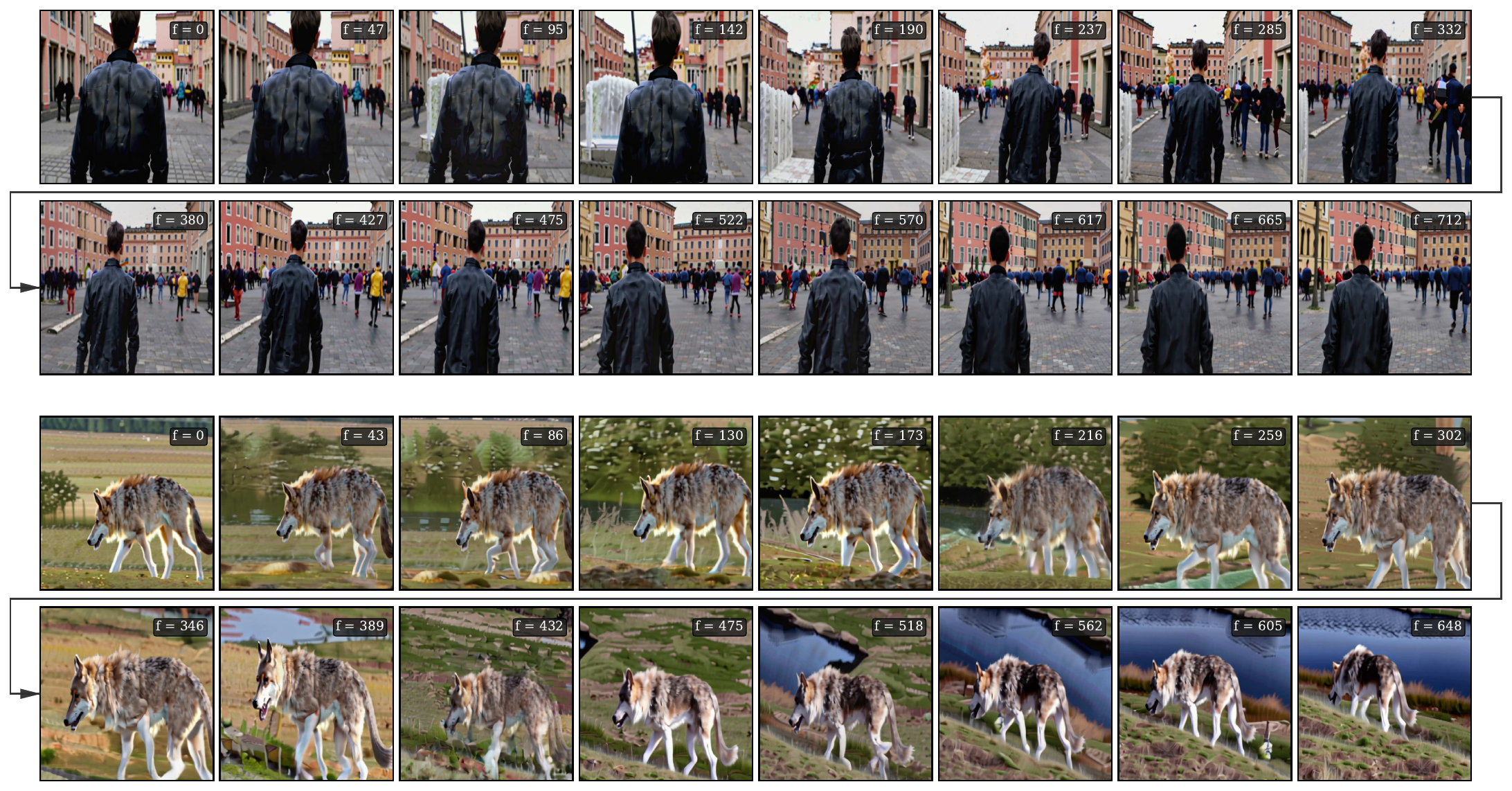}
    \caption{{\bf Chunked generation.} We demonstrate results obtained with the chunked sampling algorithm (detailed above) to generate videos $4\times$ longer than what the model was trained with. This method allows us to seamlessly rollout generations for up to a minute at $24$FPS. We write down the frame index on the top left corner of every generated frame.}
    \label{fig:t2v_chunk}
\end{figure*}

\begin{figure*}
    \centering
    \includegraphics[width=\linewidth]{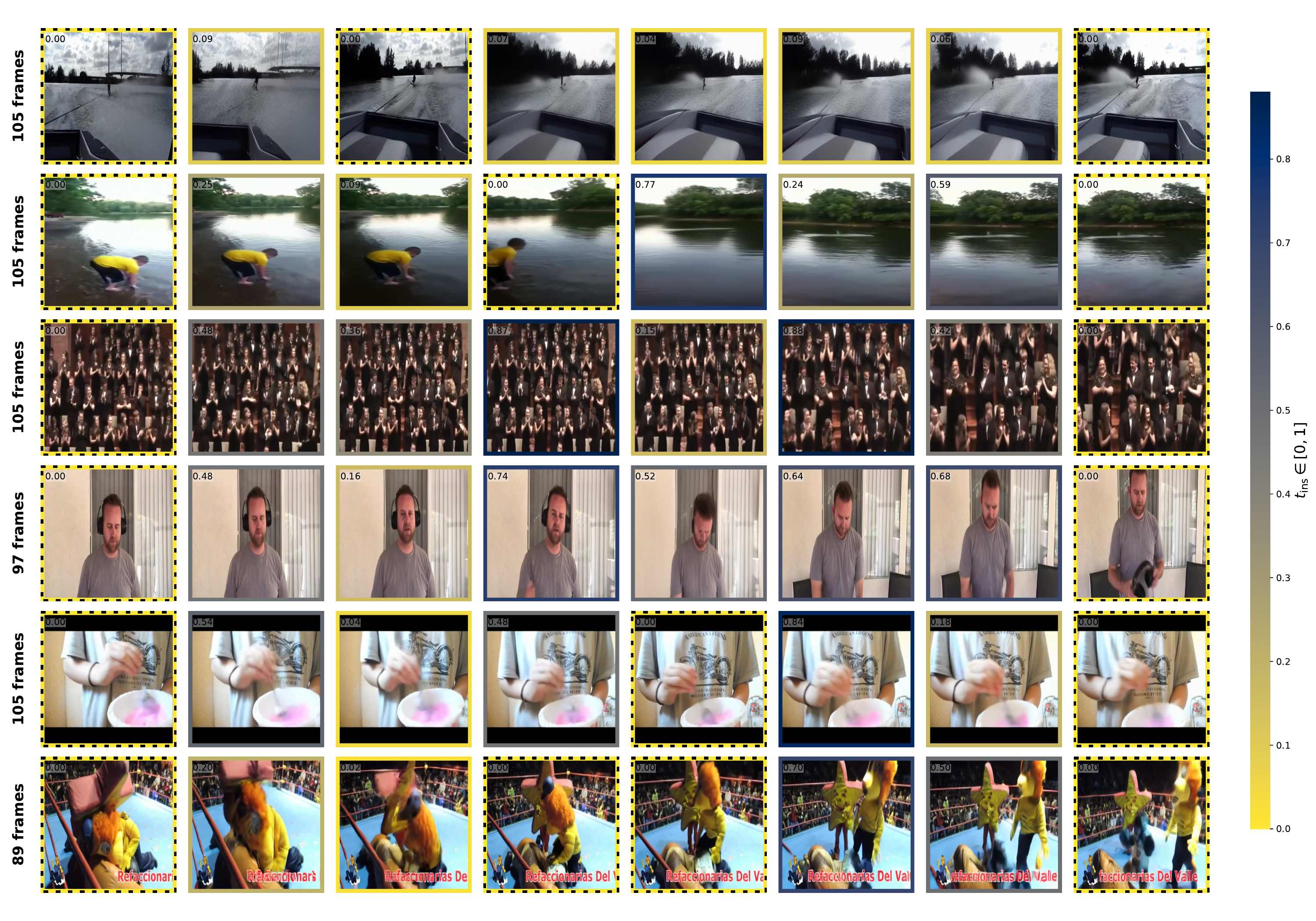}
    \caption{{\bf Additional qualitative examples on Kinetics 600 interpolation.} Each row corresponds to a different video where the first and last frame are given and up to two extra middle frames are also given. Insertion time is highlighted in the bodrer color of each frame, context frames are highlighted with dashed lines.}
    \label{fig:kinetics_2}
\end{figure*}

\begin{figure*}
    \centering
    \includegraphics[width=\linewidth]{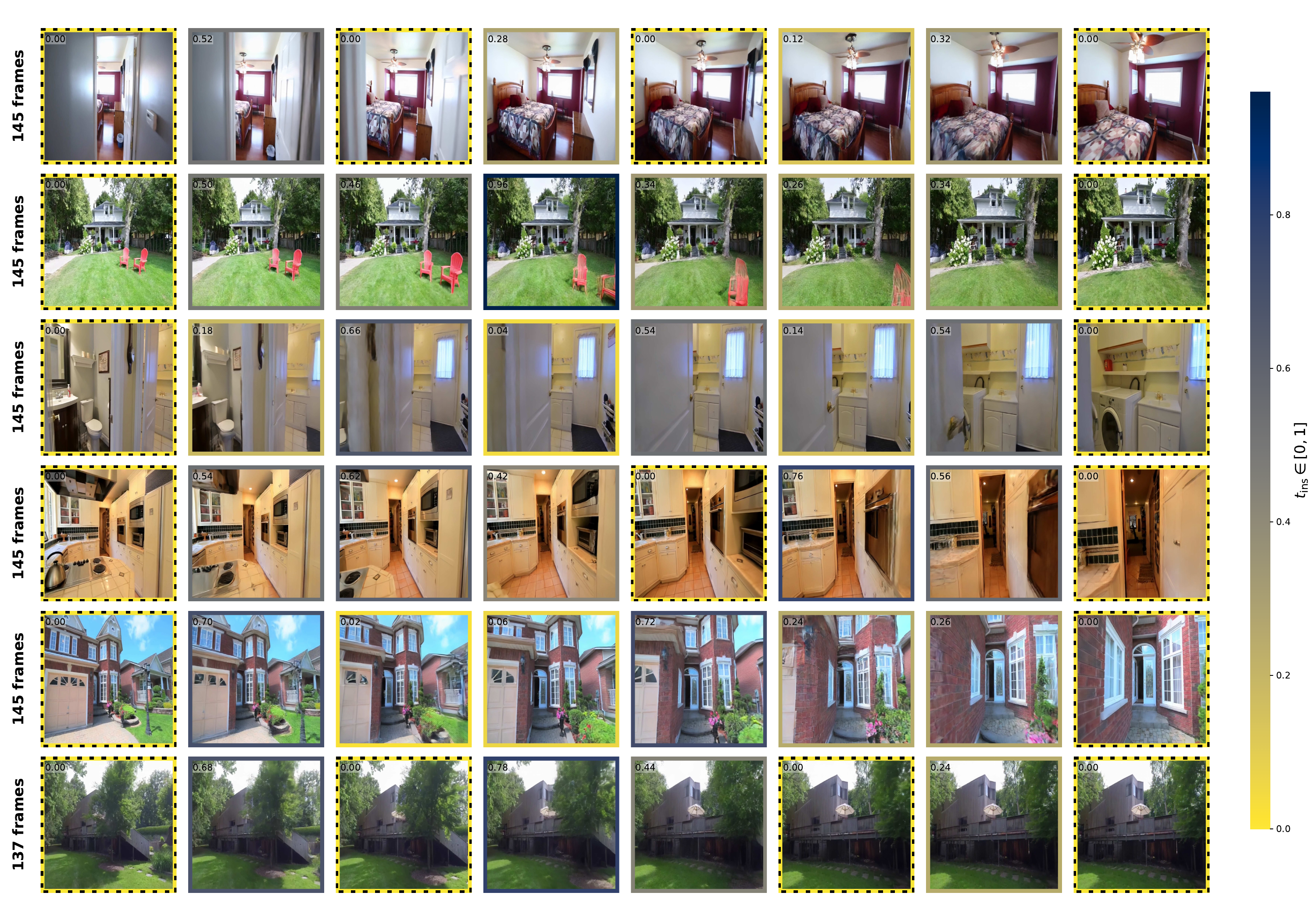}
    \caption{{\bf Additional qualitative examples on RealEstate10K interpolation.} Each row corresponds to a different video where the first and last frame are given and up to two extra middle frames are also given. Insertion time is highlighted in the border color of each frame, context frames are highlighted with dashed lines.}
    \label{fig:re10k_2}
\end{figure*}

\begin{figure*}
    \centering
    \includegraphics[width=\linewidth]{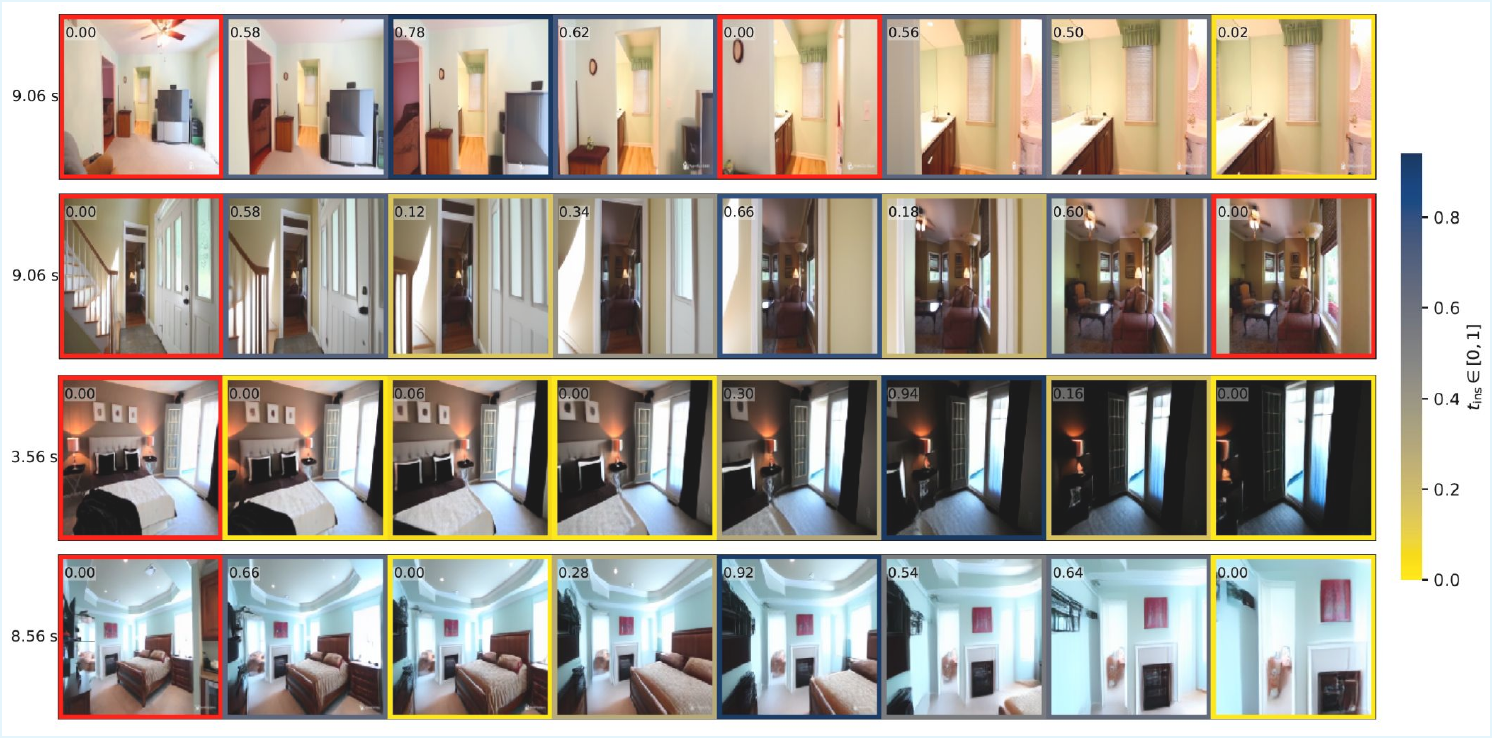}
    \caption{{\bf Qualitative examples of Image-to-Video  generation.} Using \ours trained on the RealEstate10K dataset.
    First shown frame is given as context.
    Given the initial frame, we generate videos of at most 145 frames at 16 FPS, corresponding to 9.06 secs.
    }
    \label{fig:qual_i2v_re10k}
\end{figure*}